%% file: main.tex
\documentclass[11pt]{article}
\usepackage{graphicx} 

\input{notation}

\title{Linear-Time User-Level DP-SCO via Robust Statistics}
\author{
Badih Ghazi$^1$ \quad
Ravi Kumar$^1$\and
Daogao Liu$^1$\thanks{Email:\texttt{liudaogao@gmail.com}}\quad
Pasin Manurangsi$^1$ 
\\
\textsc{$^1$Google}
}
\date{\today}

\begin{document}

\maketitle

\begin{abstract}
    User-level differentially private stochastic convex optimization (DP-SCO) has garnered significant attention due to the paramount importance of safeguarding user privacy in modern large-scale machine learning applications. 
    Current methods, such as those based on differentially private stochastic gradient descent (DP-SGD), often struggle with high noise accumulation and suboptimal utility due to the need to privatize every intermediate iterate. 
    In this work, we introduce a novel linear-time algorithm that leverages robust statistics, specifically the median and trimmed mean, to overcome these challenges. 
    Our approach uniquely bounds the sensitivity of all intermediate iterates of SGD with gradient estimation based on robust statistics, thereby significantly reducing the gradient estimation noise for privacy purposes and enhancing the privacy-utility trade-off.  
    By sidestepping the repeated privatization required by previous methods, our algorithm not only achieves an improved theoretical privacy-utility trade-off but also maintains computational efficiency.
    We complement our algorithm with an information-theoretic lower bound, showing that our upper bound is optimal up to logarithmic factors and the dependence on $\epsilon$. 
    This work sets the stage for more robust and efficient privacy-preserving techniques in machine learning, with implications for future research and application in the field.
\end{abstract}

\input{intro}
\input{prel}
\input{alg}

\input{lowerbound}
\input{discussioin}

\addcontentsline{toc}{section}{References}
\bibliographystyle{alpha}
\bibliography{ref}
\newpage

\appendix

\input{appendix}

\input{appendix_proof}
\input{appendix_lb}

\end{document}

%% file: notation.tex
\usepackage{amsmath}
\usepackage{amsthm}
\usepackage{thm-restate}
\allowdisplaybreaks
\usepackage{amssymb}

\usepackage{algorithm}
\usepackage[lined,boxed,ruled,norelsize,algo2e,linesnumbered]{algorithm2e}
\newcommand{\citep}[1]{\cite{#1}}
\newcommand{\citet}[1]{\cite{#1}}

\usepackage{subfig}

\usepackage{array}
\usepackage{color}
\usepackage{graphicx}
\usepackage{color}
\usepackage{epstopdf}
\usepackage{algpseudocode}
\usepackage{scrextend}
\usepackage[T1]{fontenc}
\usepackage{bbm}
\usepackage{comment}
\usepackage{authblk}
\usepackage{multicol}
\usepackage{dsfont}
\usepackage{mathtools}
\usepackage{enumitem}
\usepackage{mathrsfs}

\usepackage{tikz}
\usepackage{hyperref}  
\hypersetup{colorlinks=true,citecolor=blue,linkcolor=red} 

\usetikzlibrary{arrows}

\usepackage[lmargin=1in,rmargin=1in,tmargin=0.8in,bmargin=0.8in]{geometry}

\graphicspath{{./figs/}}

\definecolor{b2}{RGB}{51,153,255}
\definecolor{mygreen}{RGB}{80,180,0}

\theoremstyle{plain}
\newtheorem{theorem}{Theorem}[section]
\newtheorem{lemma}[theorem]{Lemma}
\newtheorem{proposition}[theorem]{Proposition}

\theoremstyle{definition}
\newtheorem{definition}[theorem]{Definition}
\newtheorem{assumption}[theorem]{Assumption}

\theoremstyle{remark}
\newtheorem{remark}[theorem]{Remark}

\newcommand{\wh}{\widehat}

\newcommand{\ov}{\overline}
\newcommand{\eps}{\varepsilon}
\renewcommand{\epsilon}{\varepsilon}
\renewcommand{\phi}{\varphi}

\newcommand{\calM}{\mathcal{M}}

\newcommand{\R}{\mathbb{R}}
\newcommand{\F}{\mathcal{F}}

\newcommand{\TA}{\Tilde{A}}
\newcommand{\TV}{\mathrm{TV}}

\newcommand{\bx}{\bar{x}}
\newcommand{\by}{\bar{y}}

\newcommand{\calX}{\mathcal{X}}

\newcommand{\calZ}{\mathcal{Z}}
\newcommand{\calD}{\mathcal{D}}
\newcommand{\calP}{\mathcal{P}}

\renewcommand{\hat}{\wh}
\renewcommand{\bar}{\ov}
\renewcommand{\d}{\mathrm{d}}

\DeclareMathOperator*{\E}{\mathbb{E}}

\DeclareMathOperator*{\Var}{\mathbf{Var}}

\DeclareMathOperator{\est}{est}

\DeclareMathOperator{\sign}{sign}

\DeclareMathAlphabet{\mathpzc}{OT1}{pzc}{m}{it}

\newcommand{\ind}{\mathbf{1}}

\newcommand{\reals}{\mathbb{R}} 

\newcommand{\AboTh}{\mathsf{AboveThreshold}} 
\newcommand{\Lap}{\mathrm{Lap}}
\newcommand{\qcon}{s^{\mathsf{conc}}}
\newcommand{\calN}{\mathcal{N}}
\newcommand{\calO}{\mathcal{O}}

\newcommand{\nSG}{\mathrm{nSG}}

\newcommand{\rs}{\mathrm{rs}}
\newcommand{\good}{\mathrm{good}}

\newcommand{\user}{\mathrm{user}}
\newcommand{\rmitem}{\mathrm{item}}
\newcommand{\bZ}{\bar{Z}}
\newcommand{\Error}{\mathrm{Error}}

\newcommand{\PAitem}{\mathcal{A}_{\epsilon,\delta}^{\mathrm{item}}}

\newcommand{\PAuser}{\mathcal{A}_{\epsilon,\delta}^{\mathrm{user}}}

%% file: intro.tex
\section{Introduction}
With the rapid development and widespread applications of modern machine learning and artificial intelligence, particularly driven by advancements in large language models (LLMs), privacy concerns have come to the forefront. 
For example, recent studies have highlighted significant privacy risks associated with LLMs, including well-documented instances of training data leakage \citep{carlini2021extracting,lukas2023analyzing}. These challenges underscore the urgent need for privacy-preserving mechanisms in machine learning systems.

Differential Privacy (DP) \citep{dwork2006calibrating} has emerged as a rigorous mathematical framework for ensuring privacy and is now the gold standard for addressing privacy concerns in machine learning. The classic definition of DP, referred to as \textit{item-level DP}, guarantees that the replacement of any single training example has a negligible impact on the model’s output. Formally, a mechanism $\calM$ is said to satisfy $(\epsilon,\delta)$-item-level DP if, for any pair $\calD$ and $\calD'$ of neighboring datasets that differ by a single item, and for any event $\calO \in \mathrm{Range}(\calM)$, the following condition holds: \begin{align} \label{eq:dp_def} \Pr[\calM(\calD) \in \calO] \leq e^{\epsilon} \Pr[\calM(\calD') \in \calO] + \delta. \end{align}

Item-level DP provides a strong guarantee against the leakage of private information associated with any single element in a dataset. 
However, in many real-world applications, a single user may contribute multiple elements to the training dataset \citep{Xu2024}. 
This introduces the need to safeguard the privacy of each user as a whole, which motivates a stronger notion of \textit{user-level DP}.  This notion ensures that replacing one user (who may contribute up to $m$ items) in the dataset has only a negligible effect on the model’s output. 
Formally, \eqref{eq:dp_def} still holds when $\calD$ and $\calD'$ differ by one user, i.e., up to $m$ items in total. Notably, when $m = 1$, user-level DP reduces to item-level DP.

\paragraph{DP-SCO}
As one of the central problems in privacy-preserving machine learning and statistical learning, DP stochastic convex optimization (DP-SCO) has garnered significant attention in recent years (e.g., \cite{bst14,bftt19,FKT20,bfgt20,bgm21,shw22,gopi2022private,alt24}). 
In DP-SCO, we are provided with a dataset $\{Z_i\}_{i \in [n]}$ of users, where each user $Z_i$ contributes samples drawn from an underlying distribution $\calP$. 
The goal is to minimize the population objective function under the DP constraint: 
\begin{align*} 
F(x) := \E_{z \sim \calP} f(x; z), 
\end{align*} 
where $f(x; z): \calX \to \mathbb{R}$ is a convex function defined on the convex domain $\calX$.

DP-SCO was initially studied in the item-level setting, where significant progress has been made with numerous exciting results. 
The study of user-level DP-SCO, to our knowledge, was first initiated by \cite{levy2021learning}, where the problem was considered in Euclidean spaces, with the Lipschitz constant and the diameter of the convex domain defined under the $\ell_2$-norm.
They achieved an error rate of $O\left(\frac{1}{\sqrt{mn}} + \frac{d}{n \epsilon \sqrt{m}}\right)$ for smooth functions, along with a lower bound of $\Omega\left(\frac{1}{\sqrt{mn}} + \frac{\sqrt{d}}{n \epsilon \sqrt{m}}\right)$. 

Building on this, \cite{bassily2023user} introduced new algorithms based on DP-SGD, incorporating improved mean estimation procedures to achieve an asymptotically optimal rate of $\frac{1}{\sqrt{nm}} + \frac{\sqrt{d}}{n \sqrt{m} \epsilon}$. 
However, their approach relies on the smoothness of the loss function and imposes parameter restrictions, including $n \geq \sqrt{d}/\epsilon$ and $m \leq \max\{\sqrt{d}, n\epsilon^2/\sqrt{d}\}$. 
On the other hand, \cite{GKKM+23} observed that user-level DP-SCO has low local sensitivity to user deletions. 
Using the propose-test-release mechanism, they developed algorithms applicable even to non-smooth functions and requiring only $n \geq \log(d)/\epsilon$ users. 
However, their approach is computationally inefficient, running in super-polynomial time and achieving a sub-optimal error rate of $O\left(\frac{1}{\sqrt{nm}} + \frac{\sqrt{d}}{n \sqrt{m} \epsilon^{2.5}}\right)$.  \cite{asi2023user} proposed an algorithm that achieves optimal excess risk in polynomial time, requiring only $n \geq \log(md)/\epsilon$ users and accommodating non-smooth losses. 
However, their algorithm is also computationally expensive:  
for $\beta$-smooth losses, it requires $\beta \cdot (nm)^{3/2}$ gradient evaluations, while for non-smooth losses, it requires  $(nm)^3$ gradient evaluations.
Motivated by these inefficiencies, \cite{LLA24} focused on improving the computational cost of user-level DP-SCO while maintaining optimal excess risk. For $\beta$-smooth losses, they designed an algorithm requiring $\max\{\beta^{1/4}(nm)^{9/8}, \beta^{1/2}n^{1/4}m^{5/4}\}$ gradient evaluations;  for non-smooth losses, they achieved the same optimal excess risk using $n^{11/8}m^{5/4}$ evaluations.

Linear-time algorithms have also been explored in the user-level DP-SCO setting. \cite{bassily2023user} proposed a linear-time algorithm in the local DP model, achieving an error rate of $O\left(\sqrt{d}/\sqrt{nm}\epsilon\right)$ under the constraints $m < d/\epsilon^2$, $n > d/\epsilon^2$, and $\beta \leq \sqrt{n^3/md^3}$. Similarly, \cite{LLA24} achieved the same rate under slightly relaxed conditions, requiring $n \geq \log(ndm)/\epsilon$ and $\beta \leq \sqrt{nmd}$.

In our work, we consider user-level DP-SCO under $\ell_\infty$-norm assumptions, in contrast to most of prior results, which were established under the $\ell_2$-norm. This distinction is significant, as the $\ell_\infty$-norm provides stronger per-coordinate guarantees and hence is more desirable.  Furthermore, the $\ell_\infty$-norm enjoys properties that are crucial to our algorithmic design but do not hold in the $\ell_2$-norm  setting. These properties influence both the development of our algorithm and the corresponding theoretical guarantees. The implications of this distinction and its role in our results will be discussed in detail in the subsequent sections.

\subsection{Technical Challenges in User-Level DP-SCO}

A key challenge in solving user-level DP-SCO using DP-SGD lies in obtaining more accurate gradient estimates while maintaining privacy. 
Consider a simple scenario where we perform gradient descent for $t$ steps and seek an estimate of $\nabla F(x_t)$. 
To achieve this, we sample $B$ users to estimate $\nabla F(x_t)$ and compute $q_t(Z_i):=\frac{1}{m}\sum_{z\in Z_i}\nabla f(x_t;z)$, the average of the $m$ gradients from user $Z_i$ at point $x_t$. 
If each user's $m$ functions are i.i.d. drawn from the distribution $\mathcal{P}$, then with high probability, we know that 
$
\|q_t(Z_i) - \nabla F(x_t)\| \leq \Tilde{O}(1/\sqrt{m}).
$

This naturally leads to the following mean-estimation problem: Given points $q_t(Z_1), \cdots, q_t(Z_B)$ in the unit ball, with most of them likely to be within a distance of $1/\sqrt{m}$ from each other (under the i.i.d. assumption for utility guarantees), how can we accurately and privately estimate their mean? 

A straightforward approach to recover the item-level rate is to apply the Gaussian mechanism:
\begin{align}
\label{eq:gaussian_mechanism}
\frac{1}{B} \sum_{i \in [B]} q_t(Z_i) + \calN(0, \sigma_1^2 I_d),
\end{align}
where the noise level is set as $\sigma_1 \propto 1/B$.

\paragraph{Mean Estimation and Sensitivity Control}

To improve upon this, prior works \citep{asi2023user, LLA24} designed mean-estimation sub-procedures with the following properties:
\begin{itemize}[topsep=3pt,itemsep=-3pt]
    \item \textbf{Outlier Detection:} The procedure tests whether the number of ``bad'' users (whose gradients significantly deviate from the majority) exceeds a predefined threshold (or ``break point'').
    \item \textbf{Outlier Removal and Sensitivity Reduction:} If the number of ``bad'' users is below the threshold, the procedure removes outliers and produces an estimate $g_t$ with sensitivity $\Tilde{O}(1/B\sqrt{m})$. The privatized gradient is then:
    $
    g_t + \calN(0, \sigma_2^2 I_d),
    $
    where $\sigma_2 \propto \frac{1}{B\sqrt{m}}$.
    \item \textbf{Better Variance Control:} When all users provide consistent estimates, the output follows 
    \begin{align*}
        \frac{1}{B}\sum_{i\in[B]}q_t(Z_i)+\calN(0,\sigma_2^2I_d), \sigma_2\propto \frac{1}{B\sqrt{m}},
    \end{align*}
    resulting in significantly smaller noise compared to the naive Gaussian mechanism in~\eqref{eq:gaussian_mechanism}.
\end{itemize}

By leveraging such sub-procedures, prior works have achieved the optimal excess risk rate in polynomial time. However, extending these to obtain \textit{linear-time} algorithms poses new challenges.

\paragraph{Challenges in Linear-Time User-Level DP-SCO}

Several linear-time algorithms exist for item-level DP-SCO. \cite{zhang2022differentially} and \cite{choquetteoptimal} achieved the optimal item-level rate $O\left(\frac{1}{\sqrt{n}} + \frac{\sqrt{d}}{n\epsilon}\right)$ under the smoothness assumption $\beta = O(1)$.\footnote{One can roughly interpolate the optimal item-level rate by setting $m=1$ in user-level DP.} Their approach maintains privacy by adding noise to all intermediate iterations of DP-SGD.

\cite{FKT20} notably relaxed the smoothness requirement to $\beta \leq \sqrt{n} + \sqrt{d}/\epsilon$ by analyzing the stability of non-private SGD. They showed that for neighboring datasets, the sequence $\{x_t\}_{t \in [T]}$ and $\{x_t'\}_{t \in [T]}$ remain close, ensuring that the sensitivity of the average iterate $\frac{1}{T} \sum_{t \in [T]} x_t$ is low. This allows them to apply the Gaussian mechanism directly to privatize the average iterate.

Motivated by this stability-based analysis, \cite{LLA24} attempted to generalize the linear-time approach of \cite{FKT20} to the user-level setting. However, a key difficulty arises when incorporating the mean-estimation sub-procedure. Specifically, even if one can bound $\|x_t - x_t'\|$, where $\{x_t\}$ and $\{x_t'\}$ represent the trajectories corresponding to neighboring datasets, there is no clear understanding of how applying the sub-procedure impacts stability in subsequent iterations. In particular, after performing one gradient descent step using gradient estimations from sub-procedure, we do not have guarantees on how well $\|x_{t+1} - x'_{t+1}\|$ remains bounded.

Due to this lack of stability analysis for the sub-procedure, \cite{LLA24} resorted to privatizing all iterations, resulting  in excessive Gaussian noise accumulation. Consequently, their algorithm achieved only a suboptimal error rate of $O\left(\frac{\sqrt{d}}{\sqrt{nm} \epsilon}\right)$, highlighting the fundamental challenge of designing a linear-time user-level DP-SCO algorithm by controlling the stability of the sub-procedures.

\paragraph{Generalizing Other Linear-Time Algorithms}

The linear-time algorithms proposed in \cite{zhang2022differentially} and \cite{choquetteoptimal} for item-level DP-SCO represent promising approaches to generalize  to the user-level setting. These algorithms achieve the optimal item-level rate $O\left(\frac{1}{\sqrt{n}} + \frac{\sqrt{d}}{n\epsilon}\right)$ by privatizing all intermediate iterations of variants of DP-SGD. This approach avoids the need for additional stability analysis, as the noise added at every iteration directly ensures privacy without relying on intermediate sensitivity bounds. Generalizing such algorithms to the user-level setting may, therefore, be easier compared to other approaches, as they sidestep the stability issues associated with mean-estimation sub-procedures.

In a private communication, the authors of \cite{LLA24} indicated that it is possible to generalize the linear-time algorithm of \cite{zhang2022differentially} to the user-level setting. 
However, this generalization introduces a dependence on the smoothness parameter $\beta$, which may impose restrictive smoothness constraints on the types of functions for which the algorithm is effective. 
Despite this limitation, such a generalization represents a natural direction for extending linear-time algorithms to user-level DP-SCO.

While these developments are promising, a more challenging and interesting direction lies in generalizing stability-based analyses from the item-level setting to the user-level setting. Stability-based methods, as seen in \cite{FKT20}, rely on carefully bounding the sensitivity of the entire optimization trajectory. Extending this approach to the user-level setting requires incorporating an appropriate sub-procedure for mean estimation, tailored to handle user-level sensitivity. 
This introduces additional layers of complexity, as the interactions between the sub-procedure and the iterative optimization process must be carefully analyzed to ensure stability and privacy. Overcoming these challenges would not only advance the theoretical understanding of user-level DP-SCO but might also lead to more efficient algorithms with broader applicability.

\subsection{Our Techniques and Contributions}

In this work, we design a novel mean-estimation sub-procedure based on robust statistics, such as the median and trimmed mean, specifically tailored for user-level DP-SCO. By incorporating the sub-procedure into (non-private) SGD, we establish an upper bound on $\|x_t - x_t'\|_\infty$ for all iterations $t \in [T]$. This ensures stability throughout the optimization process.

\paragraph{Key Idea: 1-Lipschitz Property of Robust Statistics}  
In one dimension, many robust statistics, such as the median, satisfy a \textit{1-Lipschitz property}. This means that if each data point is perturbed by a distance of at most $\iota$, the robust statistic shifts by at most $\iota$. This property makes robust statistics particularly well-suited for mean estimation in user-level DP-SCO.

To see this, recall that we define  
$
q_t(Z_i) := \frac{1}{m} \sum_{z \in Z_i} \nabla f(x_t; z)
$  
as the average of the $m$ gradients from user $Z_i$ at point $x_t$. Similarly, we define  
$q_t'(Z_i) := \frac{1}{m} \sum_{z \in Z_i} \nabla f(x_t'; z)$  
for the gradients at $x_t'$.  
If $|x_t - x_t'|$ is bounded, then by the smoothness of $f$, we have  
$
|q_t(Z_i) - q_t'(Z_i)| \leq \beta |x_t - x_t'|.
$  
This implies that the robust statistic computed from $\{q_t(Z_i)\}_{i \in [B]}$ and $\{q_t'(Z_i)\}_{i \in [B]}$ remains bounded by $\beta |x_t - x_t'|$ from the 1-Lipschitz property. 
As a result, the desired stability is naturally established in the one-dimensional setting.

\paragraph{Extending Stability to High Dimensions}  
In high-dimensional settings, robust statistics that satisfy the 1-Lipschitz property are not well understood. To address this, we adopt \textit{coordinate-wise robust statistics} for gradient estimation. This approach ensures stability at each coordinate level. In turn, this allows us to establish \textit{iteration sensitivity} in the $\ell_\infty$-norm.

\paragraph{Debiasing Technique}  
While robust statistics are effective in controlling sensitivity, using them directly introduces a significant bias in gradient estimation. This bias occurs even in "good" datasets where all functions are i.i.d. from the underlying distribution. If not handled properly, the bias can dominate the utility guarantee and degrade performance.

To address this issue, we propose a novel \textit{debiasing technique}: If the mean and the robust statistic are sufficiently close, we directly use the mean; otherwise, we project the mean onto the ball centered at the robust statistic.  
Both the mean and robust statistics individually satisfy the 1-Lipschitz property. We prove that this \textit{coordinate-wise projection preserves the 1-Lipschitz property in the} $\ell_\infty$-\textit{norm}. The resulting robust mean-estimation sub-procedure ensures iteration sensitivity while remaining \textit{unbiased when the dataset is well-behaved}. This property holds with high probability when all functions are i.i.d. from the distribution.

\paragraph{Improving Robust Mean Estimation: Smoothed Concentration Test}  
To further enhance stability, we introduce a \textit{smoothed version of the concentration score} for testing whether the number of ``bad'' users exceeds a threshold (the ``break point''). Prior works relied on an indicator function:  
$ 
\mathbf{1}(\|q_t(Z_i) - q_t(Z_j)\| \leq 1/\tau),
$ 
which is non-smooth and prone to instability. We replace this with a smoother function:  
$
\exp\left(-\tau \|q_t(Z_i) - q_t(Z_j)\|\right),  
$  which allows for a more stable and robust concentration test by providing a continuous measure of closeness.

\paragraph{Main Result and Implications}  
Using our sub-procedure and sensitivity bounds, we achieve a utility rate of  
\[
\Tilde{O}\left(\frac{d}{\sqrt{nm}} + \frac{d^{3/2}}{n\sqrt{m}\epsilon^2}\right),
\]  
for smooth functions defined over an $\ell_\infty$-ball, with gradients bounded in the $\ell_1$-norm and diagonally dominant Hessians (see Section~\ref{sec:prel} for detailed assumptions). We also construct a lower bound:  
\[
\Tilde{\Omega}\left(\frac{d}{\sqrt{nm}} + \frac{d^{3/2}}{n\sqrt{m}\epsilon}\right),
\]  
using the fingerprinting lemma, showing that our upper bound is nearly optimal except for the dependence on $\epsilon$.  
We discuss this loose dependence further in Section~\ref{sec:discussion}.

These assumptions are strong in terms of the properties of the norm: the $\ell_\infty$-ball is the largest among all $\ell_p$-balls and Lipschitz continuity in the $\ell_1$-norm implies that the $\ell_\infty$-norm of the gradient is bounded, which is the weakest possible assumption on gradient norms.

\paragraph{Comparison with Prior Work}  
The best-known item-level rate for this setting (i.e., Lipschitz in the $\ell_1$-norm and optimization over an $\ell_\infty$-ball) is:  
$ 
O\left(\frac{d}{\sqrt{n}} + \frac{d^{3/2}}{n\epsilon}\right),
$  
as established in the work of \cite{asi2021private}. To our knowledge, our result is the first to extend item-level rates to the user-level setting, incorporating the dependence on $m$.

Existing user-level DP-SCO results have primarily been studied in \textit{Euclidean spaces}, where functions are assumed to be Lipschitz in the $\ell_2$-norm and optimized over an $\ell_2$-ball. Since the $\ell_2$ diameter of an $\ell_\infty$-ball is $\sqrt{d}$, applying existing linear-time algorithms to our setting yields a suboptimal rate:  
$
O\left(\frac{d^{3/2}}{\sqrt{nm}\epsilon}\right).
$ 
This suboptimal dependence on $n$ arises because existing methods privatize all intermediate steps, leading to excessive noise accumulation. However, we acknowledge that our approach requires an additional assumption on the diagonal dominance of Hessians. Despite this restriction, our techniques are well-suited to the $\ell_\infty$-ball and gradient norm setting. We provide a detailed discussion of our assumptions, limitations, and open problems in Section~\ref{sec:discussion}.

\paragraph{DP and Robustness}  
There is a rich body of work exploring the connections between DP and robustness, with robust statistics playing a central role in many DP applications \citep{dwork2009differential, slavkovic2012perturbed, liu2022differential, asi2023robustness}. However, to the best of our knowledge, the application of robust statistics in private optimization remains relatively under-explored.  
We view our work as an important step in this direction and hope it inspires further research into leveraging robust statistics for private optimization.







%% file: prel.tex
\section{Preliminaries}
\label{sec:prel}
In user-level DP-SCO, we are given a dataset $\calD=\{Z_i\}_{i\in[n]}$ of $n$ users, where $Z_i \in \calZ^m$ is the user $i$'s data which consists of $m$ datapoints drawn i.i.d. from an (unknown) underlying distribution $\calP$.
The objective is to minimize the following population function under user-level DP constraint:
\begin{align*}
    F(x):=\E_{z\sim \calP}f(x;z).
\end{align*}

In this section, we present the key definitions and assumptions.  
Discussions regarding the limitations of these assumptions can be found in Section~\ref{sec:discussion}. Additional tools, including  
those from differential privacy, are deferred to Appendix~\ref{subsec:prel}.

\vspace{-2mm}
\begin{definition}[Lipschitz]
\label{def:lip}
    We say a function $f:\calX\to\R$ is \emph{$G$-Lipschitz}  with respect to $\ell_p$-norm , if for any $x,y\in\calX$, we have
    $
        |f(x)-f(y)|\le G\|x-y\|_p.$
    This means $\|\nabla f(x)\|_q\le G$ for any $x \in \calX$,
    where $1/p+1/q=1$.
\end{definition}
\vspace{-5mm}
\begin{definition}[Smooth]
\label{def:smooth}
    In this work, we say a function $f$ is \emph{$\beta$-smooth}, if $\|\nabla^2 f(x)\|_\infty\le \beta,\forall x\in\calX$, where $\|A\|_\infty:=\max_i\sum_{j}|A_{i,j}|$ for a symmetric matrix $A$.
    This implies that $\|\nabla f(x)-\nabla f(y)\|_\infty\le \beta \|x-y\|_\infty$ for any $x,y\in\calX$.
\end{definition}
\vspace{-5mm}
\begin{definition}[Diagonal Dominance]
    A matrix $A\in\R^{d\times d}$ is \emph{diagonally dominant} if 
    \begin{align*}
        |A_{i,i}|\ge \sum_{j\neq i}|A_{i,j}|,& &\forall i \in [d].
    \end{align*}
\end{definition}
\vspace{-5mm}
\begin{assumption}
\label{assum:lispchitz_smooth}
    Each function $f(;z):\calX\to \R$ in the universe is convex, $G$-Lipschitz with respect to $\ell_1$-norm (Definition~\ref{def:lip}) and $\beta$-smooth (Definition~\ref{def:smooth}).
    $\calX$ is a ball of radius $D$ in $\ell_\infty$-norm.
\end{assumption}
\vspace{-5mm}
\begin{assumption}
\label{assump:dia_dominant}
    The Hessian of each function $f(\cdot;z)$ in the universe is diagonally dominant.
\end{assumption}

Diagonal dominance, although somewhat restrictive, is a commonly discussed assumption in the literature. For example, \cite{wang2021convergence}  demonstrated the convergence rate of SGD in heavy-tailed settings under  
the assumption of diagonal dominance. Similarly, \cite{das2024towards}  studied Adam’s preconditioning effect for quadratic functions with diagonally dominant Hessians.  
In the case of 1-hidden layer neural networks (a common focus of the NTK line of work), the Hessian is diagonal (see Section 3 in \cite{liu2020linearity}). Additionally, it has been shown that in  
practice, Hessians are typically block diagonal dominant  
\cite{martens2015optimizing, botev2017practical}.  
We also discuss potential ways to avoid this assumption in  
Section~\ref{sec:discussion}.


\paragraph{Notation}
For $X\in\R^d$, we use $X[i]$ to denote its $i$th coordinate.
For a vector $X\in\R^d$ and a convex set $\calX\subset\R^d$, we use $\Pi_{\calX}(X):=\arg\min_{Y\in\calX}\|Y-X\|_2$.
For $X\in\R^d$ and $r\in\R_{\ge0}$, we use $B_\infty(X,r)$ to denote the $\ell_\infty$-ball centered at $X$ of radius $r$.

%% file: alg.tex
\section{Main Algorithm}
\label{sec:main_alg}
We present our main result in this section and explain the algorithm in a top-down manner.  The algorithm is based on the localization framework of  
\cite{FKT20}; see Algorithm~\ref{alg:loacalizatioin} in the Appendix for details. The main result is stated formally below:
\begin{theorem}
\label{thm:main_result}
Under Assumptions~\ref{assum:lispchitz_smooth} and \ref{assump:dia_dominant}, suppose $\beta\le\frac{G}{D}(\frac{\sqrt{n}\epsilon}{\sqrt{m}\log(nmd/\delta)}+\frac{\sqrt{d\log(1/\delta)\log(nmd)}}{\sqrt{m}\epsilon})$, $\epsilon\le O(1),n\ge \log^2(nd/\delta)/\epsilon$ and $ m\le n^{O(\log\log n)}$.
Setting $\eta\le\frac{D}{G}\cdot \min\{ \frac{B\sqrt{m}}{\sqrt{n}} ,   \frac{\sqrt{m}\epsilon}{\sqrt{d\log(1/\delta)\log( nmd)}}\}$, $B=100\log(mnd/\delta)/\epsilon$, $\tau=O(G\log(nmd)/\sqrt{m})$ and $\upsilon=0.9B+\frac{2\log(T/\delta)}{\epsilon}$, Algorithm~\ref{alg:loacalizatioin}  is $(\epsilon,\delta)$-user-level-DP. 
When the $nm$ functions in dataset $\calD$ are i.i.d. drawn from the underlying distribution $\calP$, it takes $mn$ gradient computations and outputs $x_S$ such that
    \begin{align*}
        \E[F(x_S)-F(x^*)]\le \Tilde{O} \left(\frac{d}{\sqrt{nm}}+\frac{d^{3/2}}{n\epsilon^2\sqrt{m}} \right).
    \end{align*}
\end{theorem}

We briefly describe the localization framework.  
In the first phase, it runs (non-private) SGD using half of the dataset, and averages the iterates to get $\bx_1$.
Roughly speaking, the solution $\bx_1$ already provides a good approximation with a small population loss when the datasets are  drawn i.i.d. from the underlying distribution. However, to ensure privacy, we require a  
sensitivity bound on $\|\bx_1\|$ and add noise $\zeta_1$ correspondingly to privatize $\bx_1$, yielding the private solution $x_1 \leftarrow \bx_1 + \zeta_1$.  

A naive bound on the excess loss due to the privatization is given by  
\[
\E[F(x_1) - F(\bx_1)] \leq G\|\zeta_1\|_2,
\]  
but the magnitude of the noise $\|\zeta_1\|_2$ is typically too large  
to achieve a good utility guarantee. Nevertheless, this process yields  
a much better initial point $x_1$ compared to the original starting  
point $x_0$. 
As a result, a smaller dataset and a smaller step size are sufficient  
to find the next good solution $\bx_2$ in expectation, with smaller noise $\|\zeta_2\|_2$ added to privatize $\bx_2$.  

This process is repeated over $O(\log n)$ phases, where each subsequent solution $\bx_S$ is progressively refined, and the Gaussian noise  
$\|\zeta_S\|_2$ becomes negligible. Ultimately, this iterative refinement  
balances privacy and utility, as established in Theorem~\ref{thm:main_result}.  
The formal argument about the utility guarantee and proof can be found in Lemma~\ref{lm:localization}.  

Our main contribution is in Algorithm~\ref{alg:dpsgd},  
which uses a novel gradient estimation sub-procedure.


\begin{algorithm2e}
\caption{SGD for User-level DP-SCO}
\label{alg:dpsgd}
\textbf{ Input:} dataset $\calD$, privacy parameters $\epsilon,\delta$, other parameters $\eta,\tau,\upsilon,B$, initial point $x_0$\;
Divide $\calD$ into {B} disjoint subsets of equal size, denoted by $\{\calD_i\}_{i\in[B]}$,
$\calD_i=\{Z_{i,t}\}_{t\in[|\calD|/B]}$\; 
Set $T=|\calD|/B$\;
\For{Step $t=1,\ldots,T$}
{
For each $i\in[B]$, get $q_t(Z_{i,t}):=\frac{1}{m}\sum_{z\in Z_{i,t}}\nabla f(x_{t-1};z)$\;
Let $g_{t-1}$ be the output of Algorithm~\ref{alg:robust_gradient_est} with inputs $\{q_t(Z_{i,t})\}_{i\in[B]}$ and threshold $1/\tau$\;
$x_{t}=\Pi_\calX(x_{t-1}-\eta g_{t-1})$
}
\tcc{Concentration Test}
\tcc{Recall the query $q_t(Z_{i,t})$ for each $t\in[T], i\in[B]$ from above}
Run Algorithm~\ref{alg:out_rem} with query $\{q_t\}_{t\in[T]}$ and parameters $\calD_t,\epsilon,\frac{\delta}{2Tmnd},\tau,\upsilon$ to get answers $\{a_t\}_{t\in [T]}$ \;
\If{$a_t=\top,\forall t\in[T]$}
{
\textbf{ Return:} Average iterate $\bar{x}=\frac{1}{T}\sum_{t\in[T]}x_t$\;
}
\Else
{
\textbf{ Output:} Initial point $x_0$\;
}
\end{algorithm2e}

\paragraph{ Iteration Sensitivity of Algorithm~\ref{alg:dpsgd}:}
The contractivity of gradient descent plays a crucial role in the sensitivity analysis, for which we need the Hessians to be diagonally  dominant
(Assumption~\ref{assump:dia_dominant}). 

\begin{restatable}{lem}{contractivity}[Contractivity]
    \label{lm:contractivity}
Suppose $f:\calX\to\R$ is a convex and $\beta$-smooth function satisfying Assumption~\ref{assump:dia_dominant}, then for any two points $x,y\in \calX$, with step size $\eta\le 2/\beta$, we have
    \begin{align*}
        \|(x-\eta \nabla f(x)) - (y-\eta \nabla f(y))\|_\infty\le \|x-y\|_\infty.
    \end{align*}
\end{restatable}

Now, we discuss Algorithm~\ref{alg:dpsgd}.  
Given the dataset $\calD$, we proceed in $T = |\calD|/B$ steps.  
At the $t$th step, we draw $B$ users $\{Z_{i,t}\}_{i \in [B]}$ and compute the average gradient of each user. 
We then apply our gradient estimation algorithm (Algorithm~\ref{alg:robust_gradient_est}) and perform normal gradient descent for $T$ steps.  

In the second phase of Algorithm~\ref{alg:dpsgd}, we perform the concentration test  
(Algorithm~\ref{alg:out_rem}) on the $B$ gradients at each step based on $\AboTh$ (Algorithm~\ref{alg:mean_est_with_AT}).  
If the concentration test passes for all steps (i.e., $a_t = \top$  
for all $t \in [T]$), we output the average iterate. Otherwise, the  
algorithm fails and returns the initial point.  
As mentioned in the Introduction, the crucial novelty of Algorithm~\ref{alg:dpsgd}  
and Algorithm~\ref{alg:robust_gradient_est} lies in bounding the sensitivity  
of each solution $x_t$ by incorporating the (coordinate-wise) robust  
statistics into SGD.



\begin{algorithm2e}
\caption{Gradient Estimation based on Robust Statistics}
\label{alg:robust_gradient_est}
\textbf{ Input:} a set of $d$-dimensional vectors $\{X_i\}_{i\in[B]}$, threshold parameter $\varsigma>0$\;
\For{Each dimension $j=1,\ldots,d$}
{
Compute the robust statistics $X_{\rs}[j]$, and the mean $\bx[j]$ over $\{X_{i}[j]\}_{i\in[B]}$\;
\If{$|X_{\rs}[j]-\bx[j]|\ge \varsigma$}
{
Set $X_{est}[j]=\Pi_{B(Y_j,\varsigma)}(\bx[j])$\;
}
\Else
{
Set $X_{est}[j]=\bx[j]$\;
}
}
\textbf{ Return $X_{est}$}
\end{algorithm2e}

We utilize robust statistics in the  
gradient estimation sub-procedure. 
We make the following assumptions regarding the robust statistics used:

\begin{assumption}
\label{assum:prop_geo_median}
    Given a set $\{X_i\}_{i \in [B]}$ of vectors, let $X_{\rs}$ be  
    any robust statistic satisfying the following properties:
    
    (i) For any $\rho \ge 0$, if there exists a point $X'$ such  
        that more than $B/2$ points lie within $B_\infty(X', \rho)$,  
        then $X_{\rs} \in B_\infty(X', \rho)$.
        
(ii) If we perturb each point $Y_i = X_i + \iota_i$ such that  
        $\|\iota_i\|_\infty \le \Delta$ for any $\Delta \ge 0$, and let  
        $Y_{\rs}$ be the robust statistic of $\{Y_i\}$, then  
        $\|X_{\rs} - Y_{\rs}\|_\infty \le \Delta$.
        
    (iii) For any real numbers $a$ and $b$, if $Z_i = aX_i + b$ for  
        each $i \in [B]$, and $Z_{\rs}$ is the corresponding robust  
        statistic of $\{Z_i\}_{i \in [B]}$, then $Z_{\rs} = aX_{\rs} + b$.  
\end{assumption}

\begin{remark}
    Common robust statistics, such as the (coordinate-wise) median and trimmed mean,  
    satisfy Assumption~\ref{assum:prop_geo_median}.
\end{remark}
\vspace{-2mm}
In Algorithm~\ref{alg:robust_gradient_est}, we output means if they are close to the robust statistics to control the bias, and project the means onto the sphere around the robust statistics in a coordinate-wise manner when they are far apart.  
However, we still need to ensure that the sensitivity remains bounded when the projection is operated.  
The following technical lemma plays a crucial role in establishing iteration sensitivity to deal with the sensitivity with potential projection operations.
\vspace{-1mm}

\begin{restatable}{lem}{projmeantors}
\label{lm:proj_mean_to_rs}
Consider four real numbers $a,b,c,d$, such that $|a-b|\le 1$, and $|c-d|\le 1$.
Let $c'=\Pi_{B(a,r)}(c)$ and $d'=\Pi_{B(b,r)}(d)$ for any $r\ge 0$.
Then, we have $|c'-d'|\le 1.$
\end{restatable}

Unfortunately, we are unaware of any robust statistic satisfying  
Assumption~\ref{assum:prop_geo_median} in high dimensions under the  
$\ell_2$-norm, and Lemma~\ref{lm:proj_mean_to_rs} does not hold in high  
dimensions either. These limitations restrict the applicability of our  
techniques in high-dimensional Euclidean spaces; see Section~\ref{sec:discussion}.  

Let $\{x_t\}_{t \in [T]}$ and $\{y_t\}_{t \in [T]}$ be two trajectories  
corresponding to neighboring datasets that differ by one user. The  
crucial technical novelty is that, for any $t \in [T]$, we can control  
$\|x_t - y_t\|_{\infty}$ as long as the number of ``bad'' users in each  
phase ($B$ in total) does not exceed the ``break point'', say $2B/3$.  
Without loss of generality, assume that $Z_{1,1} \neq Z_{1,1}'$ is the  
differing user in the neighboring dataset pairs $(\calD, \calD')$  
considered in the following proof.  

The first property of Assumption~\ref{assum:prop_geo_median} ensures that when the number of ``bad'' users in each phase does not exceed the  ``break point'' $2B/3$, the robust statistic remains close to most of the gradients, allowing us to control $\|x_1 - y_1\|_\infty$.  
To formalize this, we say that the neighboring dataset pair 
$(\calD, \calD')$ is $\rho$-\textit{aligned} if there exist points  
$X'$ and $Y'$ such that $|X_{\good}| \ge 2B/3$ and  
$|Y_{\good}| \ge 2B/3$, where  
\[
    X_{\good} = \{q_1(Z_{i,1}) : q_1(Z_{i,1}) \in B_{\infty}(X', \rho),  
    i \in [B]\},  \text{ and }
\]  
\[
    Y_{\good} = \{q_1'(Z_{i,1}') : q_1'(Z_{i,1}') \in B_{\infty}(Y', \rho),  
    i \in [B]\}.  
\]  
This definition essentially states that the number of ``bad'' users does  
not exceed the ``break point'' in either $\calD$ or $\calD'$, ensuring  
that most gradients remain well-aligned within a bounded region.

\begin{restatable}{lem}{itesensitivitybase}
    \label{lm:ite_sensitivity_base}
    For some (unknown) parameter $\rho > 0$, suppose $(\calD, \calD')$  
    is $\rho$-aligned. Then, by running Algorithm~\ref{alg:robust_gradient_est}  
    with threshold parameter $\varsigma \ge 0$, we have $\|x_1 - y_1\|_\infty \le \eta(4\rho + 2\varsigma)$.
\end{restatable}

The sequential sensitivity can be bounded using induction, with the base  
case $\|x_1 - y_1\|_\infty$ already established. The formal statement  
is provided in Lemma~\ref{lm:iteration_sensitivity}.  

\begin{algorithm2e}
\caption{Concentration Test}
\label{alg:out_rem}
\textbf{ Input:} Dataset $\calD=(Z_1,\ldots,Z_B)$, privacy parameters $\epsilon,\delta$, parameters $\tau,\upsilon$\;
\For{$t=1,\ldots,T$}
{ 
Receive a new concentration query $q_t:\calZ\to\R^d$\;
Define the concentration score
\begin{align}
\label{eq:concentration_score_def}
    \qcon_t(\calD,\tau):=\frac{1}{B}\sum_{Z\in\calD}\sum_{Z'\in \calD}\exp(-\tau\|q_t(Z)-q_t(Z')\|_\infty)\;
\end{align}
\textbf{ Return }$\AboTh(\qcon_t, \epsilon/2, \upsilon)$
}
\end{algorithm2e}

\begin{restatable}[Iteration Sensitivity]{lem}{iterationsensitivity}
\label{lm:iteration_sensitivity}  
    If we use a robust statistic satisfying Assumption~\ref{assum:prop_geo_median}  
    in Algorithm~\ref{alg:robust_gradient_est}, then for all $t \in [T]$, we have  $\|x_t - y_t\|_\infty \le \|x_1 - y_1\|_\infty$.
\end{restatable}

Lemmas~\ref{lm:ite_sensitivity_base} and \ref{lm:iteration_sensitivity}  
together establish the iteration sensitivity of Algorithm~\ref{alg:dpsgd}.

\paragraph{ Query Sensitivity of Concentration Test (Algorithm~\ref{alg:out_rem}):}
We have established iteration sensitivity for any aligned neighboring  
dataset pair $(\calD, \calD')$. Next, we analyze the influence of the  
concentration test, which we use to check if the number of ``bad'' users exceed the ``break point''.

To apply the privacy guarantee of $\AboTh$  
(Lemma~\ref{thm:Above_Threshold}), it suffices to bound the sensitivity  
of each query in the concentration test.  
Recall that we assume $Z_{1,1} \neq Z_{1,1}'$ in the neighboring datasets.  
Thus, by the definition (Equation~\eqref{eq:concentration_score_def}), it is straightforward to observe that  
\begin{align}
\label{eq:query_sensitivity_qcon_one}
    |\qcon_1(\calD, \tau) - \qcon_1(\calD', \tau)| \le 2.  
\end{align}  

Next, we consider the sensitivity of $\qcon_t$ for $t \ge 2$.  
The sensitivity is proportional to $\|x_t - y_t\|_\infty$, which we have  
already bounded by $\|x_1 - y_1\|_\infty$.  
Note that we can bound the iteration sensitivity if the neighboring  
datasets are aligned, meaning the number of ``bad'' users does not  
exceed the ``break point''. We first show that if the number of ``bad''  
users exceeds the ``break point'', the algorithm is likely to halt  
after the first step by failing the first test.

\begin{restatable}{lem}{sensitivitybase}
    \label{lm:sensitivity_base}
Suppose $B\ge \frac{100\log(T/\delta)}{\epsilon}, \epsilon\le O(1)$ and we set $\upsilon=0.9B+\frac{2\log(T/\delta)}{\epsilon}$.
Suppose for any point $Y$, we get $|X_{\good}|<B/3$ where $X_{\good}=\{q_1(Z_{i,1}):q_1(Z_{i,1})\in B_{\infty}(Y,1/\tau),i\in[B]\}$.
Then with probability at least $1-\delta/T\exp(\epsilon)$, the $\AboTh$ returns $a_1=\bot$.
\end{restatable}

We now analyze the query sensitivity between the aligned neighboring  
datasets.

\begin{restatable}[Query Sensitivity]{lem}{querysensitivity}
    \label{lm:query_sensitivity}
Suppose $6\beta\eta B\le1$.
Suppose $(\calD,\calD')$ is $(1/\tau)$-aligned and set threshold parameter $\varsigma=1/\tau$ in Algorithm~\ref{alg:mean_est_with_AT}, the sensitivity of the query is bounded by at most $2$.
That is,
\begin{align*}
    |\qcon_t(\calD,\tau)-\qcon_1(\calD',\tau)|\le 2, & & \forall t\ge 2.
\end{align*}
\end{restatable}

Equation~\eqref{eq:query_sensitivity_qcon_one} shows that the sensitivity is always bounded for $\qcon_1$.  
Lemma~\ref{lm:sensitivity_base} shows that if the number of ``bad''  
users exceeds the ``break point'', we obtain $a_1 = \bot$, and  
the query sensitivities of the subsequent queries do not need to be considered.  
Lemma~\ref{lm:query_sensitivity} establishes the query sensitivity  
in the concentration test when the neighboring datasets are aligned,  
and the number of "bad" users is below the threshold.

\paragraph{Privacy proof.}


The final privacy guarantee--stated formally below--now easily follows from the previous lemmas.
The full proof is deferred to Appendix~\ref{app:privacy-proof}.

\begin{restatable}[Privacy Guarantee]{lem}{privacyguarantee}
    \label{lm:privacy_guarantee}
    Under Assumption~\ref{assum:lispchitz_smooth} and Assumption~\ref{assump:dia_dominant}, suppose $\epsilon\le O(1), B\ge\frac{100\log(T/\delta)}{\epsilon}$, then Algorithm~\ref{alg:loacalizatioin} is $(\epsilon,\delta)$-user-level-DP.
\end{restatable}

\paragraph{Utility proof.}
We apply the localization framework in private optimization to finish the utility argument.
We analyze the utility guarantee of Algorithm~\ref{alg:dpsgd} based on the classic convergence rate of SGD on smooth convex functions (Lemma~\ref{lm:sgd_smooth}) as follows:

\begin{restatable}{lem}{dgsgdutility}
    \label{lm:dpsgd_utility}
    Let $x\in\calX$ be any point in the domain.
    Suppose the data set $\calD$ of the users, whose size $|\calD|$ is larger than $\frac{100\log(T/\delta)}{\epsilon}$, is drawn i.i.d. from the distribution $\calP$.
    Setting $\tau=G\log(nmd/\omega)/\sqrt{m}$
    then the final output $\bar{x}$ of Algorithm~\ref{alg:dpsgd} satisfies that
    \begin{align*}
        \E[F(\bar{x})-F(x)]\lesssim \left(\beta+\frac{1}{\eta} \right)\frac{\E[\|x_0-x\|^2]}{T}+\frac{\eta G^2d}{Bm}+GDd\omega.
    \end{align*}
\end{restatable}

Now we apply the localization framework.
We set $\omega=1/(nmd)^3$ to make the term depending on it negligible.
The proof of the following lemma mostly follows from \cite{FKT20}.

\begin{restatable}[Localization]{lem}{Localization}
    \label{lm:localization}
Under Assumption~\ref{assum:lispchitz_smooth} and Assumption~\ref{assump:dia_dominant}, suppose $\beta\le\frac{G}{D}(\frac{\sqrt{n}\epsilon}{\sqrt{m}\log(nmd/\delta})$, $n\ge \log^2(nd/\delta)/\epsilon, \epsilon\le O(1)$ and $ m\le n^{O(\log\log n)}$.
Set $\eta\le\frac{D}{G}\cdot \min\{ \frac{B\sqrt{m}}{\sqrt{n}} ,  \frac{\sqrt{m}\epsilon}{\sqrt{d\log(1/\delta)\log( nmd)}}\}$, $B=100\log(mnd/\delta)/\epsilon$, $\tau=O(G\log(nmd)/\sqrt{m})$ and $\upsilon=0.9B+\frac{2\log(T/\delta)}{\epsilon}$.
If the dataset is drawn i.i.d. from the distribution $\calP$,
the final output $x_S$ for Algorithm~\ref{alg:loacalizatioin} satisfies
\begin{align*}
    \E[F(x_S)-F(x^*)]\le \Tilde{O}\Big(GD\Big(\frac{d}{\sqrt{mn}}+\frac{d^{3/2}}{n\epsilon^2\sqrt{m}}\Big)\Big).
\end{align*}
\end{restatable}

\noindent\textbf{ Main Result:}
Theorem~\ref{thm:main_result} directly follows from Lemma~\ref{lm:localization} and Lemma~\ref{lm:privacy_guarantee}.



%% file: lowerbound.tex
\section{Lower Bound}
This section presents our main lower bound. As stated above, the lower bound is nearly tight, apart from lower-order terms and the dependency on $\eps$.

\begin{theorem}
\label{thm:lb}
There exists a distribution $\calP$ and a loss function $f$  satisfying Assumption~\ref{assum:lispchitz_smooth} and Assumption~\ref{assump:dia_dominant}, such that for any $(\epsilon,\delta)$-User-level-DP algorithm $\calM$, given i.i.d. dataset $\calD$ drawn from $\calP$, the output of $\calM$ satisfies
\begin{align*}
    \E[F(\calM(\calD))-F(x^*)]\ge GD\cdot \Tilde{\Omega}\Big(\min\Big\{d,\frac{d}{\sqrt{mn}}+\frac{d^{3/2}}{n\epsilon\sqrt{m}}\Big\}\Big).
\end{align*}
\end{theorem}

The non-private term $GD\frac{d}{\sqrt{mn}}$ represents the information-theoretic lower bound for SCO under these assumptions (see, e.g., Theorem 1 in \cite{agarwal2009information}).

We construct the hard instance as follows:
let $\calX=[-1,1]^d$ be unit $\ell_\infty$-ball and let $f(x;z)=-\langle x,z\rangle$ for any $x\in \calX$ be the linear function.
Let $z\in[-\sqrt{m},\sqrt{m}]^d$  with $\E_{z\sim\calP}[z]=\mu$.
Then one can easily verify that $f$ satisfies Assumptions~\ref{assum:lispchitz_smooth} and~\ref{assump:dia_dominant} with $G=\sqrt{m},D=1$ and $\beta=0$.
We have
\begin{align}
    F(\calM(\calD))-F(x^*) &= \sum_{i=1}^{d} (\sign(\mu[i])-\calM(\calD)[i])\cdot\mu[i]
   \nonumber \\ &\ge \sum_{i=1}^{d}  |\mu[i]|.\ind\big(\sign(\mu[i])\neq\sign(\calM(\calD)[i])\big). \label{eq:opt_error_to_sign_error}
\end{align}

By~\eqref{eq:opt_error_to_sign_error}, we reduce the optimization error to the weighted sign estimation error.  
Most existing lower bounds rely on the $\ell^2_2$-error of mean estimation.  
We adapt their techniques, especially the fingerprinting lemma, and provide the proof in the Appendix~\ref{sec:lbproof}.

%% file: discussioin.tex
\section{Conclusion}
\label{sec:discussion}

We present a linear-time algorithm for user-level DP-SCO that leverages (coordinate-wise) robust statistics in the gradient estimation subprocedure and provide a lower bound that nearly matches the upper bound up to logarithmic terms and an additional dependence on $\epsilon$. Despite this progress, several limitations and open problems remain, some of which we highlight below.

\begin{itemize}
    \item \textbf{Generalization to Euclidean and Other Spaces.}  
    Extending our results to Euclidean and other spaces is an interesting but technically challenging problem. One key challenge is the lack of robust statistics that are $1$-Lipschitz under perturbations in the high-dimensional $\ell_2$-norm (see the second item in Assumption~\ref{assum:prop_geo_median}).
    One may wonder whether commonly used robust statistics, such as the geometric median, are stable in this sense. However, we provide a simple counterexample involving the geometric median in the Appendix (Section~\ref{sec:counter_example}).  
    
    Another challenge arises from our approach of projecting the mean towards the robust statistic to ensure unbiased gradient estimation. This projection is $1$-Lipschitz under perturbations in one dimension (Lemma~\ref{lm:proj_mean_to_rs}), but there are known counterexamples in higher dimensions \cite[Lemma 16]{alt24}. Overcoming these issues is crucial for extending our method to general spaces.

    \item \textbf{ Additional Assumption on Diagonal Dominance.}  
    Our analysis assumes that the Hessians of functions in the universe are diagonally dominant, which is primarily used to show that gradient descent is contractive in the $\ell_\infty$-norm (Lemma~\ref{lm:contractivity}). This assumption is somewhat restrictive compared to the $\ell_2$-norm setting, where it is sufficient to assume that the operator norm of the Hessians is bounded (i.e., smoothness). Addressing the aforementioned challenges and generalizing our results to the Euclidean space could potentially eliminate this additional assumption.

    \item \textbf{ Suboptimal Dependence on $\epsilon$.}  
    Although our upper bound nearly matches the lower bound, it has a suboptimal dependence on $\epsilon$. This issue arises from the loose dependence on sensitivity. Specifically, we draw $B$ users at each step and compute their average gradient, with $B = \Tilde{O}(1/\epsilon)$. However, the final sensitivity remains roughly $\Tilde{O}(1/\sqrt{m})$ and does not improve with larger $B$. An improvement in the dependence on $\epsilon$ could be achieved if the sensitivity of the robust statistic could benefit from larger $B$.
\end{itemize}

Finally, it would be interesting to explore the use of robust statistics, such as the median used in this work, to address other private optimization problems.

%% file: appendix.tex
\newpage
\section{Preliminaries}
\label{subsec:prel}

\subsection{Differential Privacy}
\begin{definition}[User-level DP]
    We say a mechanism $\calM$ is \emph{$(\epsilon,\delta)$-user-level DP}, if for any neighboring datasets $\calD$ and $\calD'$ that differ from one user, and any output event set $\calO$, we have
    \begin{align*}
        \Pr[\calM(\calD)\in\calO]\le e^\epsilon\Pr[\calM(\calD')\in\calO]+\delta.
    \end{align*}
\end{definition}

\begin{proposition}[Gaussian Mechanism]
\label{prop:GM}
Consider a function $f:\calP^*\to\R^d$.
If $\max_{\calD\sim\calD'}\|f(\calD)-f(\calD')\|_2\le \Delta$, where $\calD\sim\calD'$ means $\calD$ and $\calD'$ are neighboring datasets, then the Gaussian mechanism
\begin{align*}
    \calM(\calD):= f(\calD)+\zeta,
\end{align*}
where $\zeta\sim\calN(0,\sigma^2I_d)$ with $\sigma^2\ge \frac{2\Delta^2\log(1.25/\delta)}{\epsilon^2}$ is $(\epsilon,\delta)$-DP.
\end{proposition}

\subsubsection{AboveThreshold}
Our algorithms use the AboveThreshold algorithm~\citep{DR14},  which is a key tool in DP to identify whether there is a query $q_i: \calZ \to \reals$ in a stream $q_1,\dots,q_T$ of queries 
 that is above a certain threshold $\Delta$. 
The $\AboTh$ algorithm (Algorithm~\ref{alg:mean_est_with_AT} presented in the Appendix) has the following guarantees:

\begin{lemma}[\cite{DR14}, Theorem 3.24]
\label{thm:Above_Threshold}
    $\AboTh$ is $(\epsilon,0)$-DP.
    Moreover, let $\alpha=\frac{8\log(2T/\gamma)}{\epsilon}$ and $\calD \in \calZ^n$. For any sequence $q_1,\cdots,q_T : \calZ^n \to \reals$ of $T$ queries each of sensitivity $1$, $\AboTh$ halts at time $k \in [T+1]$ such that with probability at least $1-\gamma$,
    \begin{itemize}
        \item For all $t < k$, $a_t =\top$ and $q_t(\calD) \ge \Delta - \alpha$;
        \item $a_k = \bot$ and $q_k(\calD) \le \Delta + \alpha$ or $k = T+1$.
    \end{itemize} 
\end{lemma}

\subsection{SubGaussian and Norm-SubGaussian Random Vectors}
\begin{definition}
Let $\zeta > 0$. We say a random vector $X$ is \emph{SubGaussian} ($\mathrm{SG}(\zeta)$) with parameter $\zeta$ if 
$\E[e^{\langle v,X-\E X\rangle}]\le e^{\|v\|^2\zeta^2/2}$ for any $v\in \R^d$.
Random vector $X\in \R^d$ is \emph{Norm-SubGaussian} with parameter $\zeta$ ($\nSG(\zeta)$) if 
$\mathbb{P}[\|X-\E X\|_2\ge t]\le 2e^{-\frac{t^2}{2\zeta^2}}$ for all $t > 0$.
\end{definition}

\begin{theorem}[Hoeffding-type inequality for norm-subGaussian, \cite{jin2019short}]
\label{thm:hoeffding_nSG}
    Let $X_1,\cdots,X_k\in\R^d$ be random vectors, and let $\F_i=\sigma(x_1,\cdot,x_i)$ for $i\in[k]$ be the corresponding filtration.
    Suppose for each $i\in[k]$, $X_i\mid \F_{i-1}$ is zero-mean $\nSG(\zeta_i)$. Then, there exists an absolute constant $c>0$, for any $\gamma>0$,
    \begin{align*}
        \mathbb{P}\left[\left\|\sum_{i\in[k]}X_i\right\|_2\ge c\sqrt{\log (d/\gamma)\sum_{i\in[k]}\zeta_i^2}\right]\le \gamma.
    \end{align*}
\end{theorem}

\begin{algorithm2e}
\caption{$\AboTh$}
\label{alg:mean_est_with_AT}
\textbf{ Input:} Dataset $\calD = (Z_1,\dots,Z_n) $, threshold $\Delta \in \reals$, privacy parameter $\epsilon$\;
Let $\hat{\Delta}:= \Delta-\Lap(\frac{2}{\epsilon})$\;
\For{$i=1$ to $T$}
{
Receive a new query $q_i: \calZ^n \to \reals$ \;
Sample $\nu_i \sim \Lap(\frac{4}{\epsilon})$\;
\If{$q_t(\calD)+\nu_i<\hat{\Delta}$}
{
\textbf{ Output:} $a_i=\bot$\;
\textbf{ Halt}\;
\Else{
\textbf{ Output:} $a_i=\top$\;
}
}
}
\end{algorithm2e}

\subsection{Optimization}
\begin{lemma}[\cite{bubeck2015convex}]
\label{lm:sgd_smooth}
Consider a $\beta$-smooth convex function $f$ over a convex set $\calX$.
For any $x\in\calX$, suppose that the random initial point $x_0$ satisfies $\E[\|x_0-x\|_2^2]\le R^2$.
Suppose we have an unbiased stochastic gradient oracle such that $\E\|\Tilde{g}(x_t)-\nabla f(x_t)\|_2^2\le \sigma_t^2$, then running SGD for $T$ steps with fixed step size $\eta$ satisfies that
\begin{align*}
    \E\left[f \left(\frac{1}{T}\sum_{t=1}^{T}x_{t+1} \right)-f(x)\right]\le \left(\beta+\frac{1}{\eta} \right)\frac{R^2}{T} + \frac{\eta\sum_t\sigma_t^2}{2T}.
\end{align*}
\end{lemma}

%% file: appendix_proof.tex
\newpage
\section{Proof of Section~\ref{sec:main_alg}}
\begin{algorithm2e}
    \caption{Localization for user-level DP-SCO}
    \label{alg:loacalizatioin}
    \textbf{Input:} Dataset $\calD$, parameters $\epsilon,\delta,B$, initial point $x_0$\;
    \textbf{Process:}
    Set $S=\lceil \log n/B\rceil $\;
    \For{Phase $s=1,\ldots,S$}
    {
    Set $n_s=n/2^s$ and $\eta_s=(\log^{-s}m)\eta$\;
    Draw a dataset $\calD_s$ of size $n_s$ from the unused users\;
    Run Algorithm~\ref{alg:dpsgd} with inputs $\calD_s,\epsilon,\delta,\eta_s,\tau,\upsilon,B,x_{s-1}$\;
    set $x_s=\bar{x}_s+\zeta_s$, where $\zeta_s\sim \calN(0,\sigma_s^2I_d)$ with $\sigma_s=O(\frac{\eta_s G\sqrt{d\log(\exp(\epsilon)/\delta)\log(nmd)}}{\sqrt{m}\epsilon})$\;
    }
     \textbf{ Return:} the final solution $x_s$
\end{algorithm2e}

\subsection{Proof of Lemma~\ref{lm:contractivity}}
\contractivity*
\begin{proof}
By the diagonal dominance assumption and precondition that $\eta\le 2/\beta$, we know
\begin{align*}
    \|I-\eta\nabla^2f(z)\|_\infty=\max_{j}\left\{|1-\eta\nabla^2 f(z)_{j,j}|+\sum_{i\neq j}\eta |\nabla^2 f(z)_{j,i}|\right\}\le 1.
\end{align*}

Note that by the mean-value theorem, 
\begin{align*}
    (x-\eta \nabla f(x)) - (y-\nabla f(y))= x-y+\eta (x-y)^\top \nabla^2 f(z) = (x-y)(I-\eta \nabla^2 f(z)),
\end{align*}
where $z$ is in the segment between $x$ and $y$.
Hence we have
\[
\| (x-\eta \nabla f(x)) - (y-\nabla f(y)) \|_\infty  
\le \|x-y\|_\infty \cdot \|I-\eta \nabla^2 f(z) \|_\infty\cdot \\
\le \|x-y\|_\infty.
\]
\end{proof}

\subsection{Proof of Lemma~\ref{lm:proj_mean_to_rs}}
\projmeantors*
\begin{proof}
    Without loss of generality, we assume $a\le b$.
    We do case analysis.

    Case (I): if no projection happens. Then it is straightforward that $c'=c$ and $d'=d$ and hence $|c'-d'|=|c-d|\le 1$.

    Case (II): if one projection happens. Without loss of generality, assume we project $c$ and get $c'$. We analyze the following sub-cases:\\ 
    (i): $c'=a-r$. In this case we know $c\le c'$.
    If $d\ge a-r=c'$, then we know $|c'-d'|\le |c-d|\le 1$.
    If $d<a-r$, then $d-b<-r$ which is impossible.\\    
    (ii): $c'=a+r$. If $a+r\le b$, then we know $|c'-d'|\le |a-b|\le 1$.
    Consider the subsubcase when $a+r>b$. If $d\le a+r$, then $|c'-d'|\le |c-d|\le 1$.
    Else if $d\ge a+r$, as $b+r\ge d\ge c'=a+r$, we have $|c'-d'|\le |a-b|\le 1$.
    
    Case (III): if two projections happen.\\
    (i): $c'=a-r,d'=b-r$. This is a trivial case.\\
    (ii): $c'=a+r,d'=b+r$. This case is also trivial.\\
    (iii): $c'=a-r,d'=b+r$. We can show that $|c'-d'|\le |c-d|\le 1$.\\
    (vi): $c'=a+r,d'=b-r$. If $a+r\le b$, then we know $|c'-d'|\le |a-b|\le 1$.
    Else, if $a+r>b$, then we know $|c'-d'|\le |c-d|\le 1$.
\end{proof}

\subsection{Proof of Lemma~\ref{lm:ite_sensitivity_base}}
\itesensitivitybase*
\begin{proof}
It suffices to show that $\|g_0-g_0'\|_\infty\le 4\rho+2\varsigma$.
By the first property of Assumption~\ref{assum:prop_geo_median} and the preconditions in the statement, we know $B_\infty(X',\rho)\cap B_\infty(Y',\rho)\neq \emptyset$,
which leads to that
\begin{align*}
    \|X'-Y'\|\le 2\rho.
\end{align*}

Moreover, we have that the robust statistic $\|X_{\rs}-Y_{\rs}\|_\infty\le 4\rho$ by the triangle inequality as $\|X_{\rs}-X'\|_\infty\le \rho$ and $\|Y_{\rs}-Y'\|_\infty\le \rho$.

By the projection operation in Algorithm~\ref{alg:robust_gradient_est}, we know $g_0\in B_\infty(X_{\rs},\varsigma)$ and $g_0'\in B_\infty(Y_{\rs},\varsigma)$, and hence we know $\|g_0-g_0'\|_\infty\le 4\rho+2\varsigma$.
This completes the proof.
\end{proof}

\subsection{Proof of Lemma~\ref{lm:iteration_sensitivity}}
\iterationsensitivity*
\begin{proof}
Recall that we assume all users $Z_{i,t}$ are equal to $Z_{i,t}'$ but $Z_{1,1}$.

We actually show that 
\begin{align*}
    \|(x_{t-1}-\eta g_{t-1}) - (y_{t-1}-\eta g_{t-1}')\|_\infty\le \|x_{t-1}-y_{t-1}\|_\infty,
\end{align*}
as the projection operator to the same convex set is contractive in $\ell_2$- and $\ell_\infty$-norm in our case.

Let $X_{i,t}=x_{t-1}-\eta q_t(Z_{i,t})$ and $Y_{i,t}=y_{t-1}-\eta q_t'(Z_{i,t})$. Note that the users used in computing the gradients are the same.
Let $X_{\est}$ be the output of Algorithm~\ref{alg:robust_gradient_est} with $\{X_{i,t}\}_{i\in[B]}$ as inputs, and $Y_{\est}$ be the corresponding output of $\{Y_{i,t}\}_{i\in[B]}$.
By the third property of Assumption~\ref{assum:prop_geo_median}, it suffices to show that 
\begin{equation}
\label{eq:contractive_est}
    \|X_{\est}-Y_{\est}\|_\infty\le \|x_{t-1}-y_{t-1}\|_\infty.
\end{equation}

By Lemma~\ref{lm:contractivity}, we know 
\begin{align*}
    \|X_{i,t}-Y_{i,t}\|_\infty\le \|x_{t-1}-y_{t-1}\|_\infty.
\end{align*}

Then by the second property in Assumption~\ref{assum:prop_geo_median}, we know that $\|X_{\rs}-Y_{\rs}\|_\infty\le \|x_{t-1}-y_{t-1}\|_\infty$.
Similarly, by Lemma~\ref{lm:contractivity}, we have $\|\bx-\by\|_\infty\le \|x_{t-1}-y_{t-1}\|_\infty$.
Then~\eqref{eq:contractive_est} follows from Lemma~\ref{lm:proj_mean_to_rs}.
This completes the proof.
\end{proof}

\subsection{Proof of Lemma~\ref{lm:sensitivity_base}}
\sensitivitybase*
\begin{proof}
The main randomness comes from the Laplacian noise we add to the query and the threshold.
Under the precondition that $|X_{\good}|<B/3$ for any $Y$, 
then we know the concentration score 
\begin{align*}
    \qcon_1(\calD,\tau)=\sum_{Z_{i,1}}\sum_{Z_{j,1}}\exp(-\tau \|q_1(Z_{i,1})-q_1(Z_{j,1}\|)
    \le 2B/3+\exp(-1)\cdot B/3<0.8B.
\end{align*}
Then by Lemma~\ref{thm:Above_Threshold} with a probability of at least $1-\delta/T\exp(\epsilon)$, we have
\begin{align*}
    \qcon_1(\calD,\tau)+\Lap(6/\epsilon)\le \upsilon,
\end{align*}
which means $a_1=\bot$.
\end{proof}

\subsection{Proof of Lemma~\ref{lm:query_sensitivity}}
\querysensitivity*
\begin{proof}
Consider the difference between $\|q_t(Z_{j,t})-q_t(Z_{i,t})\|_\infty-\|q_t'(Z_{i,t})-q_t'(Z_{j,t})\|_\infty$.

By Lemma~\ref{lm:ite_sensitivity_base} and Lemma~\ref{lm:iteration_sensitivity}, we know $\|x_t-y_t\|_\infty\le 6\eta/\tau$. 
Then by the smoothness assumption, 
we have
\begin{align*}
& ~\|q_t(Z_{j,t})-q_t(Z_{i,t})\|_\infty-\|q_t'(Z_{i,t})-q_t'(Z_{j,t})\|_\infty\\
\le &~ \|q_t(Z_{i,t})-q_t'(Z_{i,t})-(q_t(Z_{j,t})-q_t'(Z_{j,t}))\|_\infty\\
\le &~ \|q_t(Z_{i,t}-q_t'(Z_{i,t})\|_\infty+\|q_t(Z_{j,t})-q_t'(Z_{j,t})\|_\infty\\
\le &~ 2\beta\|x_t-y_t\|_\infty.
\end{align*}

Hence we have
\begin{align*}
    \qcon_i(\calD,\tau)=&\frac{1}{B}\sum_{Z,Z'\in\calD}\exp(-\tau\|q_i(Z)-q_i(Z')\|_\infty)\\
    \ge & \frac{1}{B}\sum_{Z,Z'\in\calD'}\exp(-\tau\|q_i'(Z)-q_i'(Z')\|_\infty)\exp(-12\beta\eta)\\
    \ge & \exp(-12\beta\eta)\qcon_i(\calD',\tau).
\end{align*}
As both $\qcon_i(\calD,\tau)$ and $\qcon_i(\calD',\tau)$ are in the range $[0,B]$, we know that
\begin{align*}
    \qcon_i(\calD',\tau)-\qcon_i(\calD,\tau)\le (1-\exp(-12\beta\eta))B\le 12\beta\eta B,
\end{align*}
where we use the fact $1-\exp(-x)\le x$ for any $x\ge 0$.
    
Similarly, we can bound $\qcon_i(\calD,\tau)-\qcon_i(\calD',\tau)\le 12\beta\eta B$, and complete the proof.
\end{proof}

\subsection{Proof of Lemma~\ref{lm:privacy_guarantee}}
\label{app:privacy-proof}
\privacyguarantee*
\begin{proof}
Consider the implementation over two neighboring datasets $\calD$ and $\calD'$, and we use the prime notation to denote the corresponding variables with respect to $\calD'$.
Without loss of generality, we assume the different user is used in the first phase.

To avoid confusion, let $\bx_1$ and $\bx_1'$ be the outputs of Algorithm~\ref{alg:dpsgd} with neighboring inputs $\calD_1$ and $\calD_1'$, 
$x_1$ and $x_1'$ be the outputs of Algorithm~\ref{alg:loacalizatioin} by privatizing $\bx_1$ and $\bx_1'$ with Gaussian noise respectively.

The outputs $\bx_1$ and $\bx_1'$ of  Algorithm~\ref{alg:dpsgd} depend on the intermediate random variables $\{a_t\}_{t\in [T]}$ and $\{a_t'\}_{t\in[T]}$.

As the query sensitivity for $t=1$ is always bounded, by our parameter setting and property of $\AboTh$, we know $a_1\approx_{\epsilon/2,0}a_1'$.

We do case analysis.

(i) $(\calD_1,\calD_1')$ is not $(1/\tau)$-aligned.
Then by Lemma~\ref{lm:sensitivity_base}, either $\Pr[a_1=\bot]\ge 1-\delta/2$ or $\Pr[a_1'=\bot]\ge 1-\delta/2e^{\epsilon}$.
Then by union bound and $a_1\approx_{\epsilon/2,0}a_1'$, we have
\begin{align*}
    \Pr[a_1=a_1'=\bot]\ge 1-(1+\exp(\epsilon))\delta/2e^{\epsilon}.
\end{align*}
If $a_1=a_1'=\bot$, then $\bx_1=\bx_1'$ is the initial point.
We have for any event range $\calO$,
\begin{align*}
    \Pr[x_1\in \calO] =&\Pr[x_1\in \calO\mid a_1=\bot]\Pr[a_1=\bot]+\Pr[x_1\in\calO\mid a_1=\top]\Pr[a_1=\top]\\
    \le & \Pr[x_1'\in\calO\mid a_1'=\bot] \exp(\epsilon/2)\Pr[a_1'=\bot]+ e^{\epsilon}(\delta/2\exp(\epsilon))\\
    \le & e^{\epsilon/2}\Pr[x_1'\in\calO]+ \delta/2.
\end{align*}
This completes the privacy guarantee for case(i).

(ii) $(\calD_1,\calD_1')$ is $(1/\tau)$-aligned.
In this case, by Lemma~\ref{lm:ite_sensitivity_base} and Lemma~\ref{lm:iteration_sensitivity}, we know $\|x_t-y_t\|_\infty\le 6\eta/\tau$.
Moreover, Lemma~\ref{lm:query_sensitivity} suggests that the query sensitivity is always bounded by 1.
Then by the property of $\AboTh$ (Lemma~\ref{thm:Above_Threshold}), we have $\{a_t\}_{t\in[T]}\approx_{\epsilon/2,0}\{a_t'\}_{t\in[T]}$.
We have for any event range $\calO$,
\begin{align*}
\Pr[x_1\in\calO]=& \Pr[x_1\in\calO \mid \exists t,a_t=\bot]\Pr[\exists t,a_t=\bot]  +\Pr[x_1\in\calO\mid a_t=\top,\forall t]\Pr[a_t=\top,\forall t]\\
 \le & \Pr[x_1'\in\calO \mid \exists t,a_t'=\bot]\exp(\epsilon/2)\Pr[\exists t,a_t'=\bot]+\Pr[x_1\in\calO\mid a_t=\top,\forall t]\Pr[a_t=\top,\forall t]\\
 \le& \exp(\epsilon/2)\Pr[x_1'\in\calO \mid \exists t,a_t'=\bot]\Pr[\exists t,a_t'=\bot]\\
 &~~~+(\exp(\epsilon/2)\Pr[x_1'\in\calO\mid a_t'=\top,\forall t]+\delta/2\exp(\epsilon))\exp(\epsilon/2)\Pr[a_t'=\top,\forall t]\\
 \le& \exp(\epsilon)\Pr[x_1'\in\calO]+\delta,
\end{align*}
where we use the Gaussian Mechanism (Proposition~\ref{prop:GM}) to bound $\Pr[x_1\in\calO\mid a_t=\top,\forall t]$ by $\Pr[x_1'\in\calO\mid a_t'=\top,\forall t]$.
This completes the proof.

\end{proof}

\subsection{Proof of Lemma~\ref{lm:dpsgd_utility}}
\dgsgdutility*
\begin{proof}
By the Hoeffding inequality for norm-subGaussian vectors (Theorem~\ref{thm:hoeffding_nSG}), for each $t\in[T]$ and each $i\in[B]$, we have
\begin{align*}
    \Pr\Big[\|q_t(Z_{i,t})-\nabla F(x_{t-1})\|_\infty \ge G\log(ndm/\omega)/\sqrt{m}\Big]\le 1-\omega/nm.
\end{align*}

Then conditional on the event that $\|q_t(Z_{i,t})-\nabla F(x_{t-1})\|_\infty\le \tau$ for all $i\in[B]$ and $t\in[T]$,
by our setting of parameters, we know that we pass all the concentration tests with $a_t=\top,\forall t\in[T]$, and that
\begin{align*}
    g_{t-1}=\frac{1}{B}\sum_{i\in[B]}q_t(Z_{i,t}),
\end{align*}
which means $d_{\TV}\Big(\{g_{t-1}\}_{t\in[T]},\{\frac{1}{B}\sum_{i\in[B]}q_t(Z_{i,t})\}_{t\in[T]}\Big)\le \omega$.
Note that $\E[\sum_{i\in[B]}q_t(Z_{i,t})]=\nabla F(x_{t-1})$ and $\E[\|\frac{1}{B}\sum_{i\in[B]}q_t(Z_{i,t})-\nabla F(x_{t-1})\|_2^2 ]\le G^2d/Bm$ when all functions are drawn i.i.d. from the distribution.
By the small TV distance between $g_{t-1}$ and the good gradient estimation $\frac{1}{B}\sum_{i\in[B]}q_t(Z_{i,t})$, 
it follows from Lemma~\ref{lm:sgd_smooth} that
\begin{align*}
    \E[F(\bx)-F(x)]\lesssim (\beta+\frac{1}{\eta})\frac{\E[\|x_0-x\|^2_2]}{T}+\frac{\eta G^2d}{Bm}+GDd\omega,
\end{align*}
where the last term comes from the worst value $GDd$, and the small failure probability $\omega$.
\end{proof}

\subsection{Proof of Lemma~\ref{lm:localization}}
\Localization*
\begin{proof}
Let $\bx_0=x^*$ and $\zeta_0=x_0-x^*$.
Lemma~\ref{lm:dpsgd_utility} can be used to analyze the utility concerning $\bx_s$.
As we add Gaussian noise $\zeta_s$ to $\bx_s$ in each phase, we analyze the influence of $\zeta_s$ first.

By the assumption, we know $\|\zeta_0\|_2\le D\sqrt{d}$.
Recall that by the setting that $\eta\le \frac{D}{G}\cdot\frac{\sqrt{m}\epsilon}{d\log(1/\delta)\log(nmd)}$, for all $s\ge 0$,
\begin{align*}
\E[\|\zeta_s\|_2^2]=d\sigma_s^2=\eta_s^2d\frac{Gd\log(1/\delta)\log(nmd)}{m\epsilon^2}\le D^2d\cdot \log^{-s}m.
\end{align*}
Then by Lemma~\ref{lm:dpsgd_utility}, we have
\begin{align*}
    \E[F(x_S)]-F(x^*)=&\sum_{s=1}^{S}\E[F(\bx_{s}-\bx_{s-1})]+\E[F(x_s)-F(\bx_s)]\\
    \le & \sum_{s=1}^{S} \left(\frac{\E[\|\zeta_{s-1}\|_2^2]}{\eta_sT_s}+\frac{\eta_sG^2d}{2Bm} \right)+G\E[\|\zeta_S\|_2]\\
    \le& \sum_{s=1}^{S} \left(\frac{\log m}{2} \right)^{-(i-2)}\left(\frac{D^2d}{\eta n/B}+\frac{\eta G^2d}{2Bm} \right)+\frac{GDd}{(\log m)^{\log n}}\\
    \le & \Tilde{O} \Big( GD(\frac{d}{\sqrt{nm}}+\frac{d^{3/2}}{n\sqrt{m}\epsilon^2})\Big),
\end{align*}
where we use the fact that $\frac{1}{(\log m)^{\log n}}\le 1/nm$ when $m\le n^{\log \log n}$.
\end{proof}

\subsection{Counterexample of the 1-Lipschitz of Geometric Median}
\label{sec:counter_example}
We use the counterexample from \cite{durocher2009projection} to show the geometric median is unstable in 2-dimension space.
Recall give a set of points $P$ in $\R^d$, the geometric median of $P$, denoted by $M(P)$, is
\begin{align*}
    M(P):=\arg\min_{x}\sum_{p\in P}\|x- p\|_2.
\end{align*}

Let $P=\{(0,0),(0,0),(1,0),(1,\alpha)\}$ and $P'=\{(0,0),(0,\alpha),(1,0),(1,0)\}$ with $\alpha>0$ as a very small perturbation.
But we know $M(P)=(0,0)$ and $M(P')=(1,0)$.

%% file: appendix_lb.tex
\newpage
\section{Details of Lower Bound}
\label{sec:lbproof}
Now we construct a lower bound for the weighted sign estimation error.

\paragraph{Weighted sign estimation error.}
We construct a distribution $\calP_1$ as follows: for each coordinate $k\in[d]$, we draw $\mu[k]$ uniformly random from $[-1,1]$, and $z_{i,j}[k]\sim\calN(\mu[k],m)$ i.i.d., for $i\in[n],j\in[m]$.
The objective is to minimize weighted sign estimation error with respect to $\mu$.

\begin{lemma}
\label{lm:sign_error_lb}
    Let $\epsilon\le 0.1, \delta\le 1/(dnm)$.
    For any $(\epsilon,\delta)$-user-level-DP algorithm $\calM$, 
    there exists a distribution $\calP_2$ such that $\|\E_{z\sim \calP_2}[z]\|_\infty\le 1$ and $\|z\|_\infty\le\Tilde{O}(\sqrt{m})$ almost surely, and,
    given dataset $\calD$ i.i.d. drawn from $\calP_2$, we have
    \begin{align*}
    \E_{\calD,\calM,\mu}\sum_{i=1}^{d} |\mu[i]|\cdot \ind\big(\sign(\mu[i])\neq\sign(\calM(\calD)[i])\big)\ge \Tilde{\Omega}(\frac{d^{3/2}}{n\epsilon}).
    \end{align*}
\end{lemma}

First, by the previous result, we can reduce the user-level to item-level setting.

\begin{lemma}[\cite{levy2021learning}]
\label{lm:reduction_to_item_level}
Suppose each user $Z_i,i\in[n]$ observes $z_{i,1},\ldots,z_{i,m} \stackrel{\text{i.i.d.}}{\sim}\calN(\mu,\sigma^2I_d)$ with $\sigma$ known.
For any $(\epsilon,\delta)$-User-level-DP algorithm $\calM_{\user}$, there exists an $(\epsilon,\delta)$-Item-level-DP algorithm $\calM_{\rmitem}$ that takes inputs $(\bZ_1,\ldots,\bZ_n)$ where $\bZ_i=\frac{1}{m}\sum_{j\in[m]}z_{i,j}$ and has the same performance as $\calM_{\user}$.
\end{lemma}

Hence by Lemma~\ref{lm:reduction_to_item_level}, it suffices to consider the item-level lower bound.
We prove the following lemma:

\begin{lemma}
\label{lm:sample_compelexity_lb}
Let $\{\mu[k]\}_{k\in[d]}\stackrel{i.i.d.}{\sim}[-\sigma,\sigma]$.
    Let $\{\bZ_1,\ldots,\bZ_n\}$ be i.i.d. drawn from $\calN(\mu,\sigma^2I_d)$.
    If $\calM: \R^{n\times d}\to\{-1,1\}^d$ is $(\epsilon,\delta)$-DP, and
    \begin{align*}
        \E_{\mu,\calD,\calM}\sum_{i=1}^{d} |\mu[i]|\cdot \ind\big(\sign(\mu[i])\neq\sign(\calM(\calD)[i])\big)\le \alpha\le  \frac{d\sigma}{8},
    \end{align*}
    then $n\ge \frac{\sqrt{d}}{32\epsilon}$.
\end{lemma}

By the invariant scaling, it suffices to consider the case when $\sigma=1$.
To prove Lemma~\ref{lm:sample_compelexity_lb}, we need the fingerprinting lemma:

\begin{lemma}[Lemma 6.8 in \cite{kamath2019privately}]
\label{lm:fingerprinting}
For every fixed number $p$ and every $f:\R^n\to[-1,1]$, define $g(p):=\E_{X_{1,\ldots n}\sim\calN(p,1)}[f(X)]$.
We have
    \begin{align}
    \label{eq:finger_printing_lem}
        \E_{X_{1,\ldots n}\sim\calN(p,1)}[(1-p^2)(f(X)-p)\sum_{i\in[n]}(X_i-p)]=(1-p^2)\frac{\d}{\d p}g(p).
    \end{align}
\end{lemma}

By choosing $p$ uniformly from $[-1,1]$, we have the following observation over the expectation on the RHS of Equation~\eqref{eq:finger_printing_lem}.

\begin{lemma}
\label{lm:expec_derivate}
We have
\begin{align*}
    \E_{p\sim[-1,1]}[(1-p^2)\frac{\d}{\d p}g(p)]= \E_{p}[g(p)\cdot p]. 
\end{align*}
\end{lemma}

\begin{proof}
Using integration by parts, we have
\begin{align*}
    \E_{p\sim[-1,1]}[(1-p^2)\frac{\d}{\d p}g(p)]= & \frac{1}{2}\int_{-1}^{1}(1-p^2)\frac{\d}{\d p}g(p)  \d p\\
    =&\frac{1}{2}\int_{-1}^{1}\Big(\frac{\d}{\d p}\big((1-p^2)g(p)\big)-g(p)\frac{\d}{\d p}\big(1-p^2\big)     \Big)\d p\\
    =&\int_{-1}^1 g(p)p \d p\\
    =&\E[g(p)\cdot p].
\end{align*}
\end{proof}

Now we use the fingerprinting lemma.
\begin{lemma}
One has
\begin{align*}
    \E\left[\sum_{i\in[n],k\in[d]} (1-\mu[k]^2) (\calM(\calD)[k]-\mu[k])\cdot(\bZ_i[k]-\mu[k])\right]= \E\left[ \sum_{k\in[d]}\calM(\calD)[k]\cdot\mu[k] \right].
\end{align*}

\end{lemma}
\begin{proof}
 Fix a column $k\in[d]$.
 
Construct the $f$ for our purpose. 
Define $f:\R^n\to [-1,1]$ to be 
\begin{align*}
    f(X):=\E_{\calD,\calM}[\calM(\calD^{-k}\| X)[k]].
\end{align*}
That is, $f(X)$ is the expectation of $\calM(\calD)[k]$ conditioned on $\bZ_i[k]=X_i,\forall i\in[n]$.
And define $g:[-1,1]\to[-1,1]$ to be
\begin{align*}
    g(p):=\E_{\mu^{-k},X_{1,\ldots n}\sim \calN(p,1)}[f(X)].
\end{align*}
That is $g(p)$ is the expectation of $\calM(\calD)[k]$ conditional on $\mu[k]=p$.

Now we can calculate
\begin{align*}
    &\E\Big[(1-\mu[k]^2)(\calM(\calD)[k]-\mu[k])\sum_{i\in[n]}(\bZ_i[k]-\mu[k])\Big]\\
    =& \E_{\mu[k]\sim[-1,1],X_{1,\ldots n}\sim\calN(\mu[k],1)}\Big[ (1-\mu[k]^2) (f(X)-\mu[k])\sum_{i\in[n]}(\bZ_i[k]-\mu[k])   \Big] \\
    \stackrel{(i)}{=}& \E_{\mu[k]} [g(\mu[k])\cdot \mu[k]]\\
    =& \E[\calM(\calD)[k]\mu[k]],
\end{align*}
where $(i)$ follows from Lemma~\ref{lm:fingerprinting} and Lemma~\ref{lm:expec_derivate}.
The statement follows by summation over $k\in[d]$.
 \end{proof}

A small weighted sign error means a large $\E\left[ \sum_{k\in[d]}\calM(\calD)[k]\cdot\mu[k] \right]$, as demonstrated by the following lemma:

\begin{lemma}
Let $\calM:\R^{n\times d}\to [-1,1]^d$.
    Suppose $\{\mu[k]\}\stackrel{i.i.d.}{\sim}[-1,1]$, and
    \begin{align*}
        \E_{\mu,\calD,\calM}\left[\sum_{k=1}^{d} |\mu[k]|\cdot \ind\big(\sign(\mu[k])\neq\sign(\calM(\calD)[k])\big)\right]\le \alpha,
    \end{align*}
    then we have 
    \begin{align*}
        \E\left[ \sum_{k\in[d]}\calM(\calD)[k]\cdot\mu[k] \right]\ge \frac{d}{2}-2\alpha.
    \end{align*}
\end{lemma}

\begin{proof}
Note that
\begin{align*}
    &\E\left[\sum_{k=1}^{d}|\mu[k]|\cdot \Big(\ind(\sign(\mu[k])\neq\sign(\calM(\calD)[k])) + \ind(\sign(\mu[k])=\sign(\calM(\calD)[k]) )\Big)\right] \\
    &=\E\left[\sum_{k=1}^d|\mu[k]|\right]=d/2.
\end{align*}

Moreover, one has that
\begin{align*}
    &\E\left[ \sum_{k\in[d]}\calM(\calD)[k]\cdot\mu[k] \right]\\
    =& \E\left[\sum_{k\in[d]}|\mu[k]|\cdot \Big(\ind(\sign(\mu[k])=\sign(\calM(\calD)[k]))-\ind(\sign(\mu[k])\neq \sign(\calM(\calD)[k]))\Big)\right]\\
    \ge& d/2-2\alpha.
\end{align*}
This completes the proof.
\end{proof}

It remains to show the sample complexity to achieve a large value of 
\begin{align}
\label{eq:large_correlation}
\E\left[\sum_{i\in[n],k\in[d]} (1-\mu[k]^2) (\calM(\calD)[k]-\mu[k])\cdot(\bZ_i[k]-\mu[k])\right]\ge d/2-2\alpha.    
\end{align}

Now we complete the proof of Lemma~\ref{lm:sample_compelexity_lb}.

\begin{proof}[Proof of Lemma~\ref{lm:sample_compelexity_lb}]

Let $A_i=\sum_{k\in[d]}(1-\mu[k]^2)(\calM(\calD)[k]-\mu[k])(\bZ_i[k]-\mu[k])$.
We use the DP constraint of $\calM$ to upper bound $\E[A_i]$.

Let $\calD_{\sim i}$ denote $\calD$ with $\bZ_i$ replaced with an independent draw from $\calP_1$.
Define 
\begin{align*}
    \TA_i=\sum_{k\in[d]}(1-\mu[k]^2)(\calM(\calD_{\sim i})[k]-\mu[k])(\bZ_i[k]-\mu[k]).
\end{align*}

Due to the independence between $\calM(\calD_{\sim i})$ and $\bZ_i$, we have $\E[\TA_i]=0$.
Moreover, as $|1-\mu[k]^2|\le 1$ and $|\calM(\calD_{\sim i})-\mu[k]|\le 2$, 
we have $\E[|\TA_i|]\le 2 \sqrt{\sum_{k \in [d]} \Var(\bZ_i[k])}\le 2\sqrt{d}$.

Split $A_i$ with $A_{i,+}=\max\{0,A_i\}$ and $A_{i,-}=\min\{0,A_i\}$ and split $\TA_{i}$ similarly.
By the property of DP, we know 
\begin{align*}
    \Pr[|A_{i,+}|\ge t]\le \exp(\epsilon)\Pr[|\TA_{i,+}|\ge t]+\delta,\forall t\ge 0,\\
    \Pr[|A_{i,-}|\ge t]\ge \exp(-\epsilon)\Pr[|\TA_{i,1}|\ge t]-\delta, \forall t\ge 0.
\end{align*}
Then we have
\begin{align*}
    \E[A_i]=& \int_{0}^{\infty}\Pr[|A_{i,+}|\ge t] \d t- \int_{0}^{\infty} \Pr[|A_{i,-}|\le t]\d t\\
    \le & \exp(\epsilon) \E[|\TA_{i,+}|]-\exp(-\epsilon)\E[|\TA_{i,-}|]+2\delta T+ \int_{T}^{\infty}\Pr[|A_i|\ge t]\d t\\
    \le & \E[\TA_i]+(\exp(\epsilon)-1)\E[|\TA_{i,+}|]+ (1-\exp(-\epsilon))\E[|\TA_{i,-}|]+2\delta T+\int_{T}^{\infty}\Pr[|A_i|\ge t]\d t\\
    \le & \E[\TA_i]+2\epsilon \E[|\TA_i|]+2\delta T+\int_{T}^{\infty}\Pr[|A_i|\ge t]\d t,
\end{align*}
where the last inequality used the fact that $\exp(\epsilon)-1\le 2\epsilon$ when $\epsilon\le 1/10$.

When $\delta\le 1/dn^2$ and set $T=O(\sqrt{d\log(1/\delta)})$, we have
\begin{align*}
    \E[A_i]\le 4\epsilon \sqrt{d}+ d/8n.
\end{align*}

When $\alpha\le d/8$, Equation~\eqref{eq:large_correlation} implies that
\begin{align*}
    n(4\epsilon \sqrt{d}+ d/8n)\ge d/4,
\end{align*}
which leads to
\begin{align*}
    n\ge \frac{\sqrt{d}}{32 \epsilon}.
\end{align*}
This completes the proof.
\end{proof}

It is standard to translate the sample complexity lower bound (Lemma~\ref{lm:sample_compelexity_lb}) to the error lower bound (Lemma~\ref{lm:sign_error_lb}).
We present a proof below.

\begin{proof}[Proof of Lemma~\ref{lm:sign_error_lb}]

Let $\PAitem$ be the set of item-level DP mechanisms, and let $\PAuser$ be the set of user-level DP mechanisms.

Define the error term:
\begin{align*}
\Error[\calP,\calM,n]=\E_{\calD\sim\calP^n,\calM}\sum_{i=1}^d|\mu[i]|\ind(\sign(\mu[i])\neq\sign(\calM(\calD)[i])),
\end{align*}
where $\mu=\E_{z\sim\calP}z$.

Recall that we construct the distribution $\calP_1$ as follows: for each coordinate $k\in[d]$, we draw $\mu[k]$ uniformly random from $[-1,1]$, and $z_{i,j}[k]\sim\calN(\mu[k],m)$ i.i.d., for $i\in[n],j\in[m]$.

Let $\bar{\calP}_1$ be the following:
for each coordinate $k\in[d]$, we draw $\mu[k]$ uniformly random from $[-1,1]$, and $\bZ_{i}[k]\sim\calN(\mu[k],1)$ i.i.d., for $i\in[n]$.
$\bar{\calP}_1$ is corresponding to averaging the $m$ samples for each user.
By Lemma~\ref{lm:reduction_to_item_level}, we have
\begin{align*}
    \inf_{\calM\in\PAuser}\Error[\calP_1,\calM,nm]\ge \inf_{\calM\in\PAitem}\Error[\bar{\calP}_1,\calM,n].
\end{align*}
By Lemma~\ref{lm:sample_compelexity_lb},
\begin{align*}
\inf_{\calM\in\PAitem}\Error[\bar{\calP}_1,\calM,\sqrt{d}/32\epsilon]\ge \Omega(d).
\end{align*}

Let $n^*=\Tilde{O}(\sqrt{d}/\epsilon)$.
When we have a large data set of size $n\gg n^*$, construct $\bar{\calP}_2=\frac{n^*}{n}\bar{\calP}_1+(1-\frac{n^*}{n})\calP_3$, where $\calP_3$ is a Dirac distribution at $\mathbf{0}\in\R^d$.

Hence, by a Chernoff bound, with high probability, the number of samples drawn from $\bar{\calP}_1$ in the dataset $\calD$ is no more than $O(n^*\cdot\log(nd))=\frac{\sqrt{d}}{32\epsilon}$.
Then we have
\begin{align*}
    \inf_{\calM\in\PAuser}\Error[\calP_1,\calM,nm]\ge &\inf_{\calM\in\PAitem}\Error[\bar{\calP}_2,\calM,n]\\
    \ge &\frac{n^*}{n}\inf_{\calM\in\PAitem}\Error[\bar{\calP}_1,\calM, \sqrt{d}/32\epsilon]\ge \Tilde{\Omega}(\frac{d^{3/2}}{\epsilon n}).
\end{align*}

In the precondition of $\calP_2$, we need $\|z\|_\infty\le \Tilde{O}(\sqrt{m})$ almost surely for $z\sim\calP_2$.
But $\calP_1$ involves some Gaussian distributions.
We construct $\calP_2$ by truncating the Gaussian distributions.

More specifically, for each item $z_{i,j}$ drawn from $\calN(\mu,m I_d)$, we set $z'_{i,j}[k]=\frac{z_{i,j}[k]}{\max\{1,|z_{i,j}[k]|/G \}}$ with $G=\Theta(\sqrt{m\log(mnd)})$.
In other words, we get $z'_{i,j}$ by projecting $z_{i,j}$ to $B_\infty(0,G)$.
Fixing $\mu$, we first show  
\begin{align}
\label{eq:small_mean_shift}
    \|\E_{z\sim\calP_2}[z]-\mu\|_\infty\le O(1/dn^2).
\end{align}

It suffices to consider any coordinate $k\in[d]$, and prove
\begin{align*}
    |\E_{z\sim\calP_2}[z[k]]-\mu[k]|\le O(1/dn^2).
\end{align*}
Letting $\beta=\frac{-\mu[k]+G}{\sqrt{m}}$ and $\alpha=\frac{-\mu[k]-G}{\sqrt{m}}$, we know
\begin{align*}
   |\E_{z\sim\calP_2}[z[k]]-\mu[k]|\le \frac{n^*}{n}\cdot \sqrt{m} \cdot \frac{\phi(\alpha)-\phi(\beta)}{\int_{\alpha}^{\beta}\phi(x) \d x},
\end{align*}
where $\phi$ is density function of standard Gaussian $\calN(0,1)$.
As $\mu[k]\in[-1,1]$ and $G=\Theta(\sqrt{m\log(mnd)})$, we establish Equation~\eqref{eq:small_mean_shift}.
Denote $\mu'=\E_{z\sim\calP_2}[z\mid \mu]$, that is the mean of $\calP_2$ conditional on $\mu$.

Moreover, we have
\begin{align*}
&~\inf_{\calM\in\PAuser}\E_{\mu,\calD\sim\calP_2^{mn},\calM}\sum_{i=1}^{d}|\mu'[i]|\ind(\sign(\mu'[i])\neq \sign(\calM(\calD)))\\
\ge &  \inf_{\calM\in\PAuser}\E_{\mu,\calD\sim\calP_2^{mn},\calM}\Big(\sum_{i=1}^{d}|\mu[i]|\ind(\sign(\mu[i])\neq \sign(\calM(\calD)))-d\|\mu-\mu'\|_\infty\Big)\\
\ge & \inf_{\calM\in\PAuser}\E_{\mu,\calD\sim\calP_2^{mn},\calM}\sum_{i=1}^{d}|\mu[i]|\ind(\sign(\mu[i])\neq \sign(\calM(\calD)))- O(1/n^2)\\
\stackrel{(i)}{\ge} & \inf_{\calM\in\PAuser}\E_{\mu,\calD\sim\calP_1^{mn},\calM}\sum_{i=1}^{d}|\mu[i]|\ind(\sign(\mu[i])\neq \sign(\calM(\calD)))- O(1/n^2)\\
\ge & \inf_{\calM\in\PAuser}\Error(\calP_1,\calM,nm)-O(1/n^2)\\
\ge & \Tilde{\Omega}(\frac{d^{3/2}}{\epsilon n}),
\end{align*}
where the inequality (i) follows from that we can always sample from $\calP_2$ with samples from $\calP_1$, which means  the problem over $\calP_2$ is harder than $\calP_1$.
This completes the proof.
\end{proof}

The lower bound Theorem~\ref{thm:lb} follows from Lemma~\ref{lm:sign_error_lb} and our construction of the function class, that is
\begin{align*}
    \E[F(\calM(\calD))-F(x^*)] 
    & \ge \E_{\calD,\calM,\mu}\sum_{i=1}^{d} |\mu[i]|\cdot \ind\big(\sign(\mu[i])\neq\sign(\calM(\calD)[i])\big) \\
    & \ge \Tilde{\Omega} \left(\frac{d^{3/2}}{n\epsilon} \right)= GD \cdot \Tilde{\Omega} \left(\frac{d^{3/2}}{n\epsilon \sqrt{m}} \right),
\end{align*}
given $G=\Tilde{O}(\sqrt{m})$.

%% file: main.bbl
\newcommand{\etalchar}[1]{$^{#1}$}
\begin{thebibliography}{CTW{\etalchar{+}}21}

\bibitem[AFKT21]{asi2021private}
Hilal Asi, Vitaly Feldman, Tomer Koren, and Kunal Talwar.
\newblock Private stochastic convex optimization: Optimal rates in l1 geometry.
\newblock In {\em ICML}, pages 393--403, 2021.

\bibitem[AL24]{asi2023user}
Hilal Asi and Daogao Liu.
\newblock User-level differentially private stochastic convex optimization: Efficient algorithms with optimal rates.
\newblock In {\em AISTATS}, pages 4240--4248, 2024.

\bibitem[ALT24]{alt24}
Hilal Asi, Daogao Liu, and Kevin Tian.
\newblock Private stochastic convex optimization with heavy tails: Near-optimality from simple reductions.
\newblock {\em arXiv}, 2406.02789, 2024.

\bibitem[AUZ23]{asi2023robustness}
Hilal Asi, Jonathan Ullman, and Lydia Zakynthinou.
\newblock From robustness to privacy and back.
\newblock In {\em ICML}, pages 1121--1146, 2023.

\bibitem[AWBR09]{agarwal2009information}
Alekh Agarwal, Martin~J Wainwright, Peter Bartlett, and Pradeep Ravikumar.
\newblock Information-theoretic lower bounds on the oracle complexity of convex optimization.
\newblock In {\em NIPS}, 2009.

\bibitem[BFGT20]{bfgt20}
Raef Bassily, Vitaly Feldman, Crist{\'o}bal Guzm{\'a}n, and Kunal Talwar.
\newblock Stability of stochastic gradient descent on nonsmooth convex losses.
\newblock {\em arXiv}, 2006.06914, 2020.

\bibitem[BFTT19]{bftt19}
Raef Bassily, Vitaly Feldman, Kunal Talwar, and Abhradeep~Guha Thakurta.
\newblock Private stochastic convex optimization with optimal rates.
\newblock In {\em NIPS}, pages 11282--11291, 2019.

\bibitem[BGN21]{bgm21}
Raef Bassily, Crist{\'o}bal Guzm{\'a}n, and Anupama Nandi.
\newblock Non-euclidean differentially private stochastic convex optimization.
\newblock In {\em COLT}, pages 474--499, 2021.

\bibitem[BRB17]{botev2017practical}
Aleksandar Botev, Hippolyt Ritter, and David Barber.
\newblock Practical gauss-newton optimisation for deep learning.
\newblock In {\em International Conference on Machine Learning}, pages 557--565. PMLR, 2017.

\bibitem[BS23]{bassily2023user}
Raef Bassily and Ziteng Sun.
\newblock User-level private stochastic convex optimization with optimal rates.
\newblock In {\em ICML}, pages 1838--1851, 2023.

\bibitem[BST14]{bst14}
Raef Bassily, Adam Smith, and Abhradeep Thakurta.
\newblock Private empirical risk minimization: Efficient algorithms and tight error bounds.
\newblock In {\em FOCS}, pages 464--473, 2014.

\bibitem[Bub15]{bubeck2015convex}
S{\'e}bastien Bubeck.
\newblock Convex optimization: Algorithms and complexity.
\newblock {\em Foundations and Trends{\textregistered} in Machine Learning}, 8(3-4):231--357, 2015.

\bibitem[CCGT24]{choquetteoptimal}
Christopher~A Choquette-Choo, Arun Ganesh, and Abhradeep~Guha Thakurta.
\newblock Optimal rates for o (1)-smooth dp-sco with a single epoch and large batches.
\newblock In {\em ALT}, 2024.

\bibitem[CTW{\etalchar{+}}21]{carlini2021extracting}
Nicholas Carlini, Florian Tramer, Eric Wallace, Matthew Jagielski, Ariel Herbert-Voss, Katherine Lee, Adam Roberts, Tom Brown, Dawn Song, Ulfar Erlingsson, et~al.
\newblock Extracting training data from large language models.
\newblock In {\em USENIX Security}, pages 2633--2650, 2021.

\bibitem[DASD24]{das2024towards}
Rudrajit Das, Naman Agarwal, Sujay Sanghavi, and Inderjit~S Dhillon.
\newblock Towards quantifying the preconditioning effect of adam.
\newblock {\em arXiv preprint arXiv:2402.07114}, 2024.

\bibitem[DK09]{durocher2009projection}
Stephane Durocher and David Kirkpatrick.
\newblock The projection median of a set of points.
\newblock {\em Computational Geometry}, 42(5):364--375, 2009.

\bibitem[DL09]{dwork2009differential}
Cynthia Dwork and Jing Lei.
\newblock Differential privacy and robust statistics.
\newblock In {\em STOC}, pages 371--380, 2009.

\bibitem[DMNS06]{dwork2006calibrating}
Cynthia Dwork, Frank McSherry, Kobbi Nissim, and Adam Smith.
\newblock Calibrating noise to sensitivity in private data analysis.
\newblock In {\em TCC}, pages 265--284, 2006.

\bibitem[DR14]{DR14}
Cynthia Dwork and Aaron Roth.
\newblock The algorithmic foundations of differential privacy.
\newblock {\em Foundations and Trends{\textregistered} in Theoretical Computer Science}, 9(3--4):211--407, 2014.

\bibitem[FKT20]{FKT20}
Vitaly Feldman, Tomer Koren, and Kunal Talwar.
\newblock Private stochastic convex optimization: optimal rates in linear time.
\newblock In {\em STOC}, pages 439--449, 2020.

\bibitem[GKK{\etalchar{+}}23]{GKKM+23}
Badih Ghazi, Pritish Kamath, Ravi Kumar, Raghu Meka, Pasin Manurangsi, and Chiyuan Zhang.
\newblock On user-level private convex optimization.
\newblock {\em arXiv}, 2305.04912, 2023.

\bibitem[GLL22]{gopi2022private}
Sivakanth Gopi, Yin~Tat Lee, and Daogao Liu.
\newblock Private convex optimization via exponential mechanism.
\newblock In {\em COLT}, pages 1948--1989, 2022.

\bibitem[JNG{\etalchar{+}}19]{jin2019short}
Chi Jin, Praneeth Netrapalli, Rong Ge, Sham~M Kakade, and Michael~I Jordan.
\newblock A short note on concentration inequalities for random vectors with subgaussian norm.
\newblock {\em arXiv}, 1902.03736, 2019.

\bibitem[KLSU19]{kamath2019privately}
Gautam Kamath, Jerry Li, Vikrant Singhal, and Jonathan Ullman.
\newblock Privately learning high-dimensional distributions.
\newblock In {\em COLT}, pages 1853--1902, 2019.

\bibitem[LKO22]{liu2022differential}
Xiyang Liu, Weihao Kong, and Sewoong Oh.
\newblock Differential privacy and robust statistics in high dimensions.
\newblock In {\em COLT}, pages 1167--1246, 2022.

\bibitem[LLA24]{LLA24}
Andrew Lowy, Daogao Liu, and Hilal Asi.
\newblock Faster algorithms for user-level private stochastic convex optimization.
\newblock In {\em NeurIPS}, 2024.

\bibitem[LSA{\etalchar{+}}21]{levy2021learning}
Daniel Levy, Ziteng Sun, Kareem Amin, Satyen Kale, Alex Kulesza, Mehryar Mohri, and Ananda~Theertha Suresh.
\newblock Learning with user-level privacy.
\newblock {\em NeurIPS}, pages 12466--12479, 2021.

\bibitem[LSS{\etalchar{+}}23]{lukas2023analyzing}
Nils Lukas, Ahmed Salem, Robert Sim, Shruti Tople, Lukas Wutschitz, and Santiago Zanella-B{\'e}guelin.
\newblock Analyzing leakage of personally identifiable information in language models.
\newblock In {\em S \& P}, pages 346--363, 2023.

\bibitem[LZB20]{liu2020linearity}
Chaoyue Liu, Libin Zhu, and Misha Belkin.
\newblock On the linearity of large non-linear models: when and why the tangent kernel is constant.
\newblock {\em Advances in Neural Information Processing Systems}, 33:15954--15964, 2020.

\bibitem[MG15]{martens2015optimizing}
James Martens and Roger Grosse.
\newblock Optimizing neural networks with kronecker-factored approximate curvature.
\newblock In {\em International conference on machine learning}, pages 2408--2417. PMLR, 2015.

\bibitem[SHW22]{shw22}
Jinyan Su, Lijie Hu, and Di~Wang.
\newblock Faster rates of private stochastic convex optimization.
\newblock In {\em ALT}, pages 995--1002, 2022.

\bibitem[SM12]{slavkovic2012perturbed}
Aleksandra Slavkovic and Roberto Molinari.
\newblock Perturbed m-estimation: A further investigation of robust statistics for differential privacy.
\newblock In {\em Statistics in the Public Interest: In Memory of Stephen E. Fienberg}, pages 337--361. Springer, 2012.

\bibitem[WGZ{\etalchar{+}}21]{wang2021convergence}
Hongjian Wang, Mert Gurbuzbalaban, Lingjiong Zhu, Umut Simsekli, and Murat~A Erdogdu.
\newblock Convergence rates of stochastic gradient descent under infinite noise variance.
\newblock {\em Advances in Neural Information Processing Systems}, 34:18866--18877, 2021.

\bibitem[XZ24]{Xu2024}
Zheng Xu and Yanxiang Zhang.
\newblock Advances in private training for production on-device language models.
\newblock \url{https://research.google/blog/advances-in-private-training-for-production-on-device-language-models/}, 2024.
\newblock Google Research Blog.

\bibitem[ZTC22]{zhang2022differentially}
Qinzi Zhang, Hoang Tran, and Ashok Cutkosky.
\newblock Differentially private online-to-batch for smooth losses.
\newblock In {\em NeurIPS}, 2022.

\end{thebibliography}
